\documentclass[10pt,journal,compsoc]{IEEEtran}
\usepackage{amsmath,amsfonts}
\usepackage{amssymb}
\usepackage{algorithmic}
\usepackage{algorithm}
\usepackage{array}
\usepackage{textcomp}
\usepackage{stfloats}
\usepackage{url}
\usepackage{verbatim}
\usepackage{graphicx}
\usepackage[numbers]{natbib}
\usepackage{threeparttable}

\usepackage{xcolor}
\usepackage{hyperref}
\usepackage{cleveref}
\usepackage{bookmark}
\RequirePackage{amsthm}
\usepackage{multirow}
\RequirePackage{bm}

\usepackage{orcidlink}
\usepackage[export]{adjustbox}
\usepackage{makecell}
\usepackage{booktabs}
\usepackage{array}
\usepackage{url}
\usepackage{dsfont}
\usepackage{mathtools}
\usepackage{empheq}
\usepackage[shortlabels]{enumitem}
\newtheoremstyle{mytheoremstyle} 
    {\topsep}                    
    {\topsep}                    
    {\itshape}
    {}                           
    {\bfseries}                   
    {.}                          
    {.5em}                       
    {}  

\theoremstyle{mytheoremstyle}
\newtheorem{theorem}{Theorem}

\newtheorem{lemma}[theorem]{Lemma}

\newcommand{\ab}{\mathbf{a}}

\newcommand{\bx}{\bm{x}}

\newcommand{\bI}{\bm{I}}

\newcommand{\cN}{\mathcal{N}}

\newcommand{\EE}{\mathbb{E}}

\newcommand{\bmu}{\bm{\mu}}

\newcommand{\bSigma}{\bm{\Sigma}}

\newcommand{\tr}{\mathop{\mathrm{tr}}}

\newcommand{\diag}{{\rm diag}}
\newcommand{\rd}{{\,\mathrm{d}}}

\usepackage{xspace}
\newcommand{\dpVAE}{\texttt{dp-VAE}\xspace}

\begin{document}

\title{\color{black} Distance-Preserving Representations \\ for Genomic Spatial Reconstruction}

\author{Wenbin~Zhou\orcidlink{0009-0003-2757-8304}
        and~Jin-Hong~Du\orcidlink{0000-0001-9683-4146}
\IEEEcompsocitemizethanks{
\IEEEcompsocthanksitem
Wenbin Zhou is with the Heinz College of Information Systems and Public Policy and Machine Learning Department, Carnegie Mellon University, Pittsburgh, PA 15213, USA.
\IEEEcompsocthanksitem Jin-Hong Du is with the Musketeers Foundation Institute of Data Science and the Department of Statistics and Actuarial Science, University of Hong Kong, Hong Kong SAR.}
}

\markboth{}
{Zhou \MakeLowercase{\textit{et al.}}: Distance-Preserving Genomics Representations}


\IEEEtitleabstractindextext{%
\begin{abstract}
    The spatial context of single-cell gene expression data is crucial for many downstream analyses, yet often remains inaccessible due to practical and technical limitations, restricting the utility of such datasets.
    In this paper, we propose a generic representation learning and transfer learning framework \dpVAE, capable of reconstructing the spatial coordinates associated with the provided gene expression data. 
    Central to our approach is a distance-preserving regularizer integrated into the loss function during training, ensuring the model effectively captures and utilizes spatial context signals from reference datasets.
    During the inference stage, the produced latent representation of the model can be used to reconstruct or impute the spatial context of the provided gene expression by solving a constrained optimization problem.
    We also explore the theoretical connections between distance-preserving loss, distortion, and the bi-Lipschitz condition within generative models.
    Finally, we demonstrate the effectiveness of \dpVAE in different tasks involving training robustness, out-of-sample evaluation, and transfer learning inference applications by testing it over 27 publicly available datasets. This underscores its applicability to a wide range of genomics studies that were previously hindered by the absence of spatial data.
\end{abstract}

}

\maketitle

\begin{figure*}[!t]
  \centering
  \includegraphics[width=0.8\linewidth]{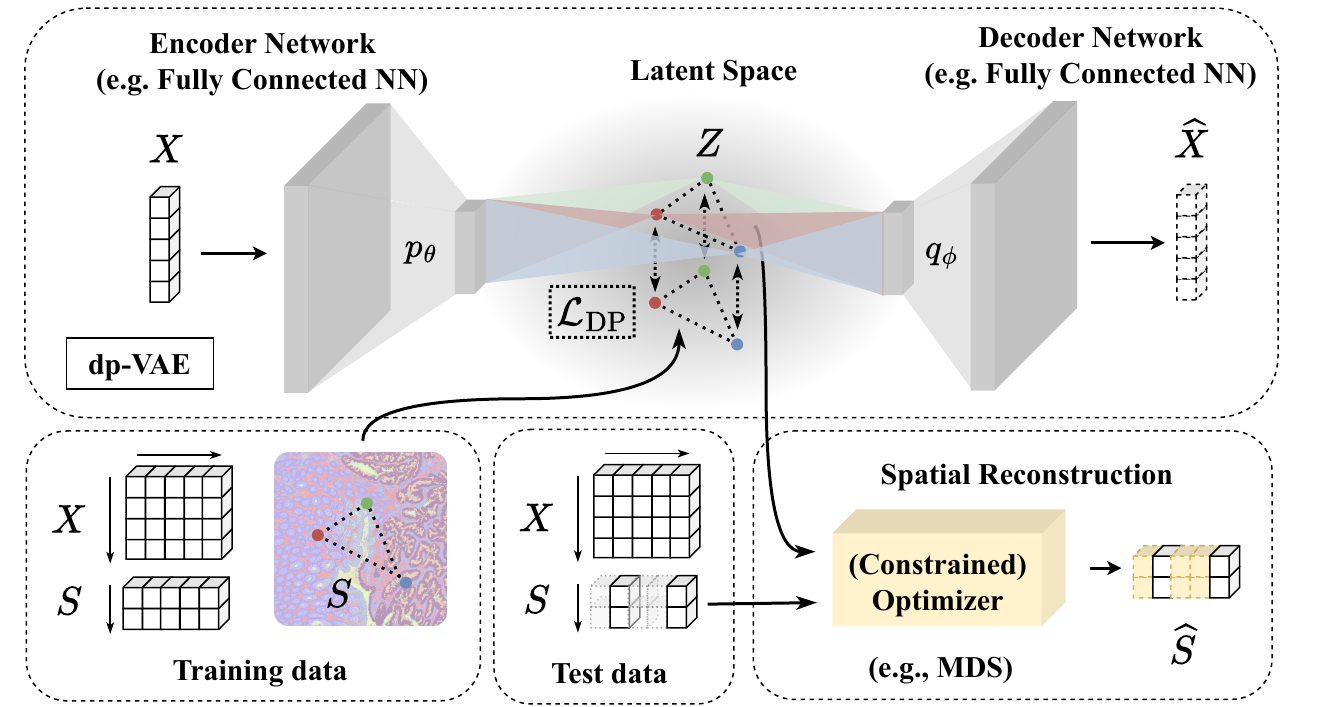}
  \caption{
  The model architecture of \dpVAE with vanilla VAE as its backbone.
  During the training stage, the distance-preserving loss ($\mathcal{L}_{\textrm{DP}}$) enforces distance preservation between the spatial domain ($S$) and the latent space ($Z$), enforcing the encoder network (\textit{e.g.}, parameterized as a fully connected neural network) to capture meaningful signals of the tissue spatial characteristics of the provided gene expression ($X$).
  During the inference/test stage, the extracted latent representations are fed to an optimizer, such as multidimensional scaling (MDS), if it is a spatial reconstruction task instead of a spatial context imputation task.
  }
  \label{fig:pipeline}
\end{figure*}

\section{Introduction}

\IEEEPARstart{U}{nderstanding} the tissue spatial context from which the gene expressions are extracted is crucial to many applications, deepening the understanding of cellular organizations and their interactions within tissues, as well as illuminating biological insights in neuroscience, developmental biology, precise diagnostics, personalized therapies, and a range of disease mechanisms such as cancer \citep{yu2022spatial}. 
As a result, measurement technologies such as spatial transcriptomics have emerged as a powerful tool in modern biology, bridging the gap between single-cell genomics and the larger physiological processes at play in living tissues.

However, various practical and technical constraints limit the range of spatial context and the granularity that these measurement tools can directly capture.
For example, spatial transcriptomics (ST) technologies such as Visium by 10x Genomics and NanoString GeoMx \citep{dries2021advances} typically do not have single-cell resolution \cite{chen2023spatial}, making the extracted spatial coordinates within the tissue space too coarse for many precision-focused downstream analyses.
Enhancing resolution to this degree requires specialized equipment, intricate protocols, and substantial computational infrastructure, thereby limiting its availability to highly specialized institutions and facilities.
On the other hand, the usual single-cell RNA sequencing (scRNA-seq) is capable of extracting gene expression at a localized single-cell level but lacks spatial contexts.
Given observations of the complementary nature of the pros and cons of these two existing mainstream gene expression extraction technologies, the following question can be asked:
\begin{quote}
    \it
    Does there exist an alternative solution to reconstruct or impute spatial context information from gene expression data?
\end{quote}
An approach towards solving this question can potentially avoid the need for the deployment of spatial transcriptomics, making the tasks we intend to carry out with spatial context more accessible.

As an initial step toward tackling this question, we propose a generic representation learning and transfer learning framework enabling the inference of spatial context from purely gene expression data (\textit{e.g.}, collected from scRNA-seq) by drawing reference from datasets where gene expression is paired with spatial information (\textit{e.g.}, collected from spatial transcriptomics). This procedure requires little or even no reliance on ST technologies, depending on the type of reference data that the user intends to use, which can even come from existing publicly available databases.
Specifically, our representation learning framework is based on the variational autoencoder (VAE) models \cite{vae}, consisting of an encoder and a decoder network and trained on the reference datasets. Instead of using the standard {\color{black} evidence lower bound (ELBO)} loss, we propose to regularize it with a distance-preserving loss function, which utilizes the spatial context as auxiliary information to enforce the learned representation to be geometrically similar to the reference dataset.
This regularization helps the model capture useful marker genes and signals from gene expressions that are representative of their spatial contexts.
During inference, since spatial context is not used as the input to the model, the model can be applied to gene expression datasets without spatial context information, but may share similar expression and geometric patterns with the reference dataset, thereby effectively extrapolating to the task of spatial reconstruction and imputation of the missing spatial context in these inference datasets.
Next, the spatial reconstruction task can be formulated as an optimization that reversely solves for the optimal coordinate sets aligning with the distance geometry of extracted latent representations from the input gene expression data. This optimization can be bypassed when we set the latent representation space to match the dimensionality of the tissue's spatial domain. An illustration of the pipeline is provided in \Cref{fig:pipeline}.

{\noindent \bf Contributions.}
The contributions of this paper are summarized as follows:
\begin{enumerate}
    \item We propose \dpVAE, a novel variational-inference-based representation learning framework, capable of extracting useful signals of the spatial context from gene expression data. Its uniqueness lies in leveraging spatial context information exclusively during the training phase, thereby enabling spatial reconstruction and imputation of gene expression datasets, which is a task underexplored by existing studies.
    \item We conduct an extensive empirical evaluation of the proposed method's robustness over 27 publically available datasets on our method's training stability, out-of-sample performance and evaluation, and transfer learning (out-of-distribution) inference capacity, providing a comprehensive analysis of the effectiveness of our method. 
    \item Additionally, we establish the theoretical connection between the distance-preserving loss and the distortion or bi-Lipschitz condition, contributing to the scarce literature on characterizing the mathematical notion of smoothness of distance-preserving mappings in the context of generative models.
\end{enumerate}
Our approach enables researchers to bridge the gap between standard gene expression profiles and spatial analyses, paving the way for more scalable and cost-effective research that integrates spatial insights into a variety of downstream applications.

\section{Related Works}

Spatial transcriptomics analysis has consistently garnered interest from the statistical machine learning community. It includes different tasks such as cellular deconvolution, clustering, visualization, cell classification, batch correction, gene imputation, and differential expression \cite{ma2022spatially, long2023spatially, lopez2018deep, du2020model}.
Classical statistical models such as marked point process \cite{edsgard2018identification}, Gaussian process \cite{svensson2018spatialde}, generalized linear spatial models \cite{sun2020statistical, zhu2021spark}, and autoregression \cite{ma2022spatially} have been deployed for these modeling tasks, where spatial context information is typically integrated into these methods in the form of covariance matrices.
In the meantime, machine learning models empowered with deep neural networks can instead directly leverage the spatial context information.
For example, convolution neural network-based methods such as CoSTA \cite{xu2021costa}, and/or graph-based methods using graph neural networks such as GLISS \cite{zhu2020integrative}, SpaGCN \cite{hu2021spagcn}, GraphST \cite{long2023spatially}. See \cite{zeng2022statistical} for a survey of these methods.

Our work aligns more closely with literature that explores how various generative models (\textit{e.g.} autoencoders) can be applied to learning efficient representation from ST and scRNA-seq, such as STAGATE \cite{dong2022deciphering}, scVSC \cite{wang2024scvsc}, scVI \cite{lopez2018deep}, and gimVI \cite{lopez2019joint}.
Some studies have also proposed spatial reconstruction methodologies of tissue spatial structure, such as Tangram \cite{biancalani2021deep}, and NovoSpaRc \cite{moriel2021novosparc}. 
However, to the best of our knowledge, our work is the first to propose to bridge representation learning with tissue spatial reconstruction, drawing reference from out-of-sample and out-of-distribution datasets.
Additionally, the previously mentioned statistical and machine learning models dedicated to spatial transcriptomics are also inapplicable to such tasks, as the spatial contexts are required as a conditional input to the models during the inference stage.  

The distance-preserving loss used in our study has (concurrently) emerged in parallel domains within the broader machine-learning literature. For instance, distance-preserving mappings have been introduced for GANs for unsupervised domain mapping \cite{benaim2017one}, image-to-audio transformation \cite{port2020earballs, kim2020face}. It has also appeared as a regularizer for VAEs \cite{chen2022local} and adversarial learning \citep{hadsell2006dimensionality, mao2019metric}. A stronger notation compared to distance-preserving, named isometry \cite{yonghyeon2021regularized, lee2021neighborhood, beshkov2022isometric} has also been extensively studied in machine learning models such as VAEs.
Building on top of these works, our derivation of the distance-preserving loss's connection to distortion or the bi-Lipschitz condition contributes to formalizing this notion.

\section{Problem setup}
\label{sec:problem-def}

At the high level, the goal of representation learning is to learn a mapping that extracts latent embeddings $z \in \mathbb{R}^d$ from given gene expression $x \in \mathbb{R}^D$. In our setting, we additionally assume that we have access to some spatial context information $s \in \mathbb{R}^{2}$ associated with each gene expression, indicating where the cell is located in the tissue slice. It is only available in the training stage, but not during inference. 

We expect that the learned representation learning model can also undergo some procedure to reconstruct the spatial domain given only gene expression during the inference stage, which we refer to as the \textit{spatial reconstruction} task.
Formally, given the gene expression a set of gene expressions denoted as $x_i \in \mathbb{R}^D$, $i = 1, \ldots, N$, the goal of spatial reconstruction is to obtain a mapping to infer the spatial coordinates $s_i \in \mathbb{R}^2$ for each cellular location $i$.
In some cases, the spatial information of some cell has already been observed, indexed by $i \in \mathcal{I} \subseteq \left\{ 1, \ldots, N \right\}$, then the spatial reconstruction task turns into the \textit{spatial imputation} task, which aims to infer the rest of the spatial coordinates $s_i$, where $i \in \left\{ 1, \ldots, N \right\} / \mathcal{I}$.
{\color{black}
The model should be able to handle both spatial reconstruction and imputation tasks.
Given their similarity, we use the two terms interchangeably, referring to them as spatial reconstruction in most cases throughout this paper.
}

We remark that, in practice, both spatial reconstruction and imputation are useful for many downstream analyses, such as spatial domain identification, region-region interaction analysis, and spatial trajectory analysis.
The inference task is similar in spirit to the setting of transfer learning \cite{pan2009survey}. 

\section{Method}

To tackle the aforementioned challenges, we propose a gene expression representation learning framework. The model, named distance-preserving variation autoencoder (\dpVAE), incorporates the spatial information into the learning objective as a penalty for the error of the pairwise similarity between the model's learned latent variable and the spatial information of the training data. In this case, the spatial information can also be viewed as an auxiliary variable during the training stage rather than a direct input of the model, therefore it is not required during the inference stage\footnote{Whereas conditional models such as CVAE \cite{sohn2015learning, wu2023counterfactual} do require the spatial context as inputs, therefore cannot be applied in this setting.}. During inference, spatial reconstruction, and spatial imputation, this penalty can be used reversely to decipher the spatial coordinates from the extracted latent representations of the data.  

In this section, we give a thorough description of our method by first introducing the background knowledge of variational autoencoders in \Cref{sec:vae}, then we describe the proposed \dpVAE in \Cref{sec:dvae}. Next, we formulate the spatial reconstruction and imputation tasks as a tractable optimization problem. Finally, we comment on the theoretical connections between this distance-preserving loss and the distortion and bi-Lipschitz condition in \Cref{sec:theory}.

\subsection{Background: Variational Autoencoder}
\label{sec:vae}

Variational autoencoders (VAEs) are representation-learning models that have been proven to be capable of effectively extracting representations from gene expression data.  
The architecture of a traditional VAE includes an encoder-decoder model designated by the posterior probability $p_\theta(z | x)$ and the likelihood $q_\phi(x | z)$, where $x$ is the gene expression, and $z$ is the latent variable that lives in low-dimensional space $\mathbb{R}^{m'}$ ($m' \leq m$). 
The objective of VAE is to estimate the parameterized marginal distribution $p_\phi(x)$ of the gene expression
\begin{equation}
    \label{eq:lvm}
    p_\phi(y) = \int p_\phi(y | z) p(z) d z,
\end{equation}
where $p(z)$ serves as a predefined prior distribution for the latent variable.
Given the computational inefficiencies linked to the normalizing constant in likelihood calculations, VAEs prioritize maximizing the evidence lower bound (ELBO) obtained via variational inference as opposed to directly optimizing \eqref{eq:lvm}:
\begin{equation}
    \label{eq:elbo}
    \underbrace{
    \mathbb{E}_{z \sim q_\theta(z|x)} \log p_\phi(x | z)}_{- \mathcal{L}_{\rm recon}} - \underbrace{D_{\rm KL}\left(q_\theta(z | x) \Vert p(z) \right)}_{\mathcal{L}_{\rm KL}}.
\end{equation}
This ELBO format is proven to be a lower bound to the marginal likelihood $\log p_\phi(y)$; the derivation is included in the appendix.
We refer to the negative of the first term as the reconstruction loss \footnote{To see why it is called the reconstruction term, if setting the likelihood and the posterior distribution to Gaussian distributions, the term is simplified to $\Vert \widehat{x}_i - x_i \Vert^2$, which is the standard autoencoder reconstruction loss.} $\mathcal{L}_{\rm recon}$, and the negative of the second term as the KL divergence $\mathcal{L}_{\rm KL}$.
Optimizing \eqref{eq:elbo} involves tuning parameters $\theta$ and $\phi$, essentially conducting representation learning. This process enables the encoder network $p_\theta(x)$ to effectively translate signals from an ambient space $\mathbb{R}^D$ into a fundamental subspace $\mathbb{R}^d$ of lower dimensions.
  
However, the existing encoding paradigm presented in \eqref{eq:elbo} reveals implicit biases, undermining its efficacy. Two primary concerns highlight this limitation:
($i$) The likelihood-focused optimization objective, while capturing some meaningful signals, underscores probabilistic aspects that could be incorrectly specified due to cross-dependencies or non-identical distributed genomic samples. This could lead to entangled and biased representations of spatial context being learned from the training data;
($ii$) In the meantime, directly inserting the spatial information as part of the input variable violates our assumption about the setting where spatial context is not available as input during the inference stage, but rather should be the output of the model.
Therefore, there is a need to develop a new method that incorporates likelihood-based generative modeling strategies with spatial information.

\subsection{Distance-Preserving Variational Autoencoder}\label{sec:dvae} 

We propose to leverage the spatial information by incorporating it in the objective function as a distance-preserving regularization term of the VAE model. First, we define the \textit{distance-preserving loss}
\begin{equation}
    \label{eq:dp}
    \mathcal{L}_{\rm DP} = \frac{1}{n^2} \sum_{i, j} \left\vert \Vert z_{i} -  z_{j} \Vert - \lambda \Vert s_{i} - s_{j} \Vert \right\vert,
\end{equation}
where $\left\Vert \cdot \right\Vert$ is the Euclidean norm.
The intuition behind the design of this loss is to ensure that the latent space preserves the relative pairwise distances between data points in the original space, while the scaling factor $\lambda$ can be adjusted appropriately to account for differences in the variance of the underlying prior distributions of the latent and original spaces.
This encourages the model to maintain geometric consistency while learning meaningful representations of gene expression aligned with the objective of VAE.
Next, we incorporate this loss as a weighted regularizer in the objective function of the training model
\begin{equation}
\label{eq:objective}
    \alpha_1 \left( \mathcal{L}_{\rm recon} + \beta \mathcal{L}_{\rm KL} \right) + {\alpha_2 \mathcal{L}_{\rm DP}},
\end{equation}
where the reconstruction term $\mathcal{L}_{\rm recon}$ and KL divergence term $\mathcal{L}_{\rm KL}$ are as defined in \eqref{eq:elbo}. The outer-level regularization coefficients $\alpha_1, \alpha_2 \leq 0$ control for the numerical balance of the distance-preserving loss and the ELBO loss, and $\beta$ is the temperature coefficient for the KL-divergence term \cite{higgins2017beta}.
The model is trained by learning the parameters of the encoder and decoder networks $\theta$ and $\phi$ as well as the scaling parameter $\lambda$ to minimize \eqref{eq:objective}.

\begin{figure*}[!ht]
    \centering
    \includegraphics[width=0.8\linewidth]{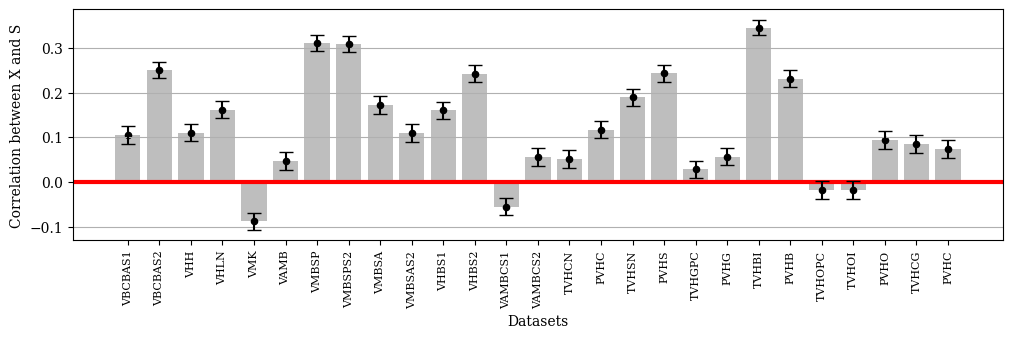}
    \caption{Pearson correlation between the pairwise distance of gene expression ($X$) and the pairwise distance ($S$) of spatial location, computed on 27 spatial transcriptomics datasets. The red line indicates zero correlation, and the error bars are the $95\%$ confidence intervals of the correlation coefficient calculated using Fisher's Z transformation.   
    }
    \label{fig:corr-X-S}
\end{figure*}

The purpose of introducing the distance-preserving penalty term can be viewed from two perspectives:
\begin{itemize}
    \item[($i$)] From the biological perspective, neighboring cells, often from similar tissue types, tend to exhibit homogeneity in attributes like gene expression and biological functions. Representation-based generative models, which aim to capture such spatial relationships and tissue-level patterns, should utilize such geometric structures in learning disentangled representations that are meaningful.
    Prior works \cite{hu2021spagcn,sun2020statistical} have illustrated that proximity in physical distance between cells has a strong association with their gene expressivity.
    This is also the initial of many spatial transcriptomics machine learning models to incorporate spatial context as inputs to enhance the modeling of cellular behavior and tissue organization, enabling more robust, context-aware insights into biological systems. 
    If viewed reversely, it is also reasonable to claim that the learned gene expression representations can also reveal the associated spatial context from where they are located in the tissue.
    \item[($ii$)] From the modeling perspective, the penalization term directs the model to focus on capturing useful signals from the gene expression, which might serve as marker genes for its spatial information.
\end{itemize}
It is noted that the first point is also supported by a simple experiment we conducted, illustrated in \Cref{fig:corr-X-S}. It can be seen that out of the 27 datasets that we are using (see the experiment section for a detailed description), 23 datasets show a statistically significant ($95\%$) positive correlation between the spatial distance of the sample locations and the gene expression similarity. This highlights the potential of seeing the spatial context as a useful source of information to aid the learning of gene expression representations, or, conversely, gene expression representations are capable of revealing information about their associated spatial context.
{
\color{black}
Though we want to note that, despite the statistical significance shown in \Cref{fig:corr-X-S}, the magnitude of the correlations differs across datasets, with some close to zero. This means that the theoretical maximum amount of information that can be extracted from the association could be weak for certain datasets, which could lead to high variance and instability for poorly constructed estimation models. This further motivates our study, which aims at constructing a powerful machine learning architecture that exploits such information for accurate spatial reconstruction and imputation. 
}

\begin{figure}[!ht]
    \hspace{-7ex}
    \includegraphics[width=1.3\linewidth]{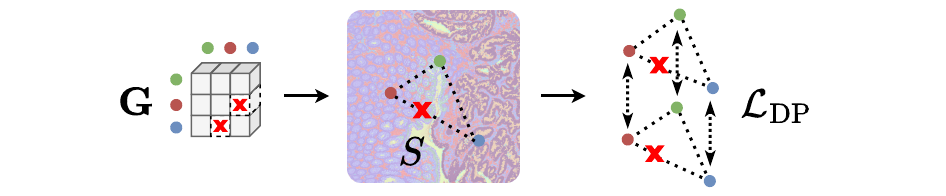}
    \caption{An illustrative example of the masking procedure is shown. Suppose we have three spatial locations, so the mask matrix $\mathbf{G}$ is $3 \times 3$. One user-specified connection in $S$ is masked during training by setting the corresponding entries in $\mathbf{G}$. Consequently, the distance-preserving regularization in $\mathcal{L}_{\rm DP}$ is not enforced between those two locations.}
    \label{fig:masking}
\end{figure}

In practice, we recommend choosing a balanced penalization coefficient $\alpha_1$ and $\alpha_2$ to strike a good balance between the statistical accuracy (controlled by the reconstruction and KL term) and the spatial signal preservation (distance preserving term), achieving strong interpolation and generalization across different datasets.
Also, we recommend a masked version of the distance-preserving:
\begin{equation}
    \label{eq:stress}
    \widetilde{\mathcal{L}}_{\rm DP}
    = \frac{1}{n^2} \sum_{i, j} \mathbf{G}_{i, j} \left\vert \Vert z_{i} -  z_{j} \Vert - \lambda \Vert s_{i} - s_{j} \Vert \right\vert,
\end{equation}
where $\mathbf{G}$ is a $\left\{0, 1\right\}^{n \times n}$ mask matrix that denotes the spatial connectivity of the cells, prespecified by the user. 
This mask allows the incorporation of biological knowledge to selectively emphasize spatial relationships that are most relevant for model training.
For example, if one has prior knowledge about the association between cells indexed by $i$ and $j$, then $\mathbf{G}_{i, j}$ can be set to one. The matrix $\mathbf{G}$ can also be specified by algorithms such as gene expression clustering, allowing it to generalize to different tasks where spatial information is present for model training. An illustration of the masking procedure is provided in \Cref{fig:masking}.

\subsection{Spatial Reconstruction and Imputation}\label{sec:recon}

In the next step, we aim to carry out the spatial reconstruction and imputation task, where we only have inference data consisting of gene expression but not spatial context information.
Assuming that we have access to some gene expression data that are paired with spatial context information and are statistically similar to the inference data, then \dpVAE can be trained on a reference dataset, but \textit{transferred} to the inference data for application.   
Specifically, since \dpVAE encodes the spatial context patterns in the latent representation via a distance-preserving regulated mapping, we can reversely find the optimal spatial coordinates in the spatial domain that align with the distance patterns in the extracted latent space. This can be formulated as an optimization problem regarding the estimated spatial coordinates. Additionally, if our task is imputation, meaning that some spatial coordinates may be exactly available, the solution to the estimation is, therefore, an optimization problem with constraints that these exact coordinates must match.

Formally, the procedure is as follows:
First, we obtain the latent embeddings $z_i$ for each gene expression in the inference dataset indexed $i = 1, \ldots, N$ by passing them through the encoder network $p_\theta$ of the fitted \dpVAE. Then, solve the following optimization problem to obtain the estimated spatial information $\widehat{s}_i$,
\begin{align}
    \min_{\widehat{s}_i} \quad & \sum_{i, j} \left\vert \Vert z_{i} -  z_{j} \Vert - \Vert \widehat{s}_{i} - \widehat{s}_{j} \Vert \right\vert \label{eq:opt} \\
    \mathrm{s.t.} \quad & \widehat{s_i} = s_i, \forall i \in \mathcal{I},\label{eq:const}
\end{align}
where $\mathcal{I}$ denotes the index set of the entries whose exact spatial coordinate is known in the imputation task.
Intuitively, the objective \eqref{eq:opt} enforces the pairwise distance of the estimated spatial information to be consistent with the latent embeddings, and the constraint \eqref{eq:const} requires that the estimated spatial information be the same for all known spatial coordinates.
The problem can be relaxed to an unconstrained optimization problem by modifying the objective function to a penalized version: $\sum_{i, j} \left\vert \Vert z_{i} -  z_{j} \Vert - \Vert \widehat{s}_{i} - \widehat{s}_{j} \Vert \right\vert + \gamma \sum_{i \in \mathcal{I}} | \widehat{s_i} - s_i |$ for easier solving, where $\gamma$ denotes the regularization coefficient (\textit{i.e.} Lagrange multiplier).

The objective function \eqref{eq:opt} is highly similar to the distance-preserving loss defined in \eqref{eq:dp}.
Observe \eqref{eq:dp} encodes the distance-preserving property into the learned latent embedding, while \eqref{eq:opt} extracts the embedded distance-preserving property to reconstruct spatial information. This can be viewed as an implicit second set of encoder-decoder-structure, which we specifically utilize, aside from the first set of encoder-decoder-structure in VAE.
However, a subtle difference between the two losses is that the parameters to be optimized in \eqref{eq:opt} and \eqref{eq:dp} are fundamentally different, where \eqref{eq:opt} is parameterized by the spatial information $s_i$, and \eqref{eq:dp} is parametrized by the parameters in the encoder network $\theta$.

\begin{figure}[!h]
    \centering
    \includegraphics[width=0.9\linewidth]{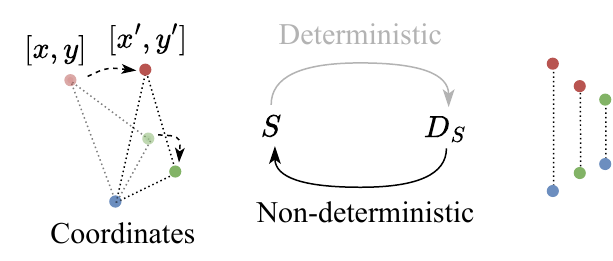}
    \caption{Illustration of forward and inverse mappings. The mapping that recovers the coordinates $S$ from pairwise distances $D_S$ is not unique or non-deterministic when $\left| \mathcal{I} \right| < 3$.}
    \label{fig:mapping}
\end{figure}

The cardinality of the observed spatial information index set $\mathcal{I}$ plays a critical role in the solution of the optimization problem.
The optimal solution to \eqref{eq:opt} is only unique when $\left| \mathcal{I} \right| \geq 3$ but otherwise not, as outlined in \Cref{fig:mapping}.
When $\left| \mathcal{I} \right| = 0$, the problem boils down to the classical multidimensional scaling (MDS) \cite{saeed2018survey} objective, which is a well-studied dimensionality reduction problem and can be solved via, \textit{e.g.}, the SMACOF algorithm \cite{leeuw1977application}.

Additionally, when $\left| \mathcal{I} \right| = 0$, the optimization procedure can also be skipped:
{
\color{black}
By directly setting the latent space dimension to $d=2$, the learned embedding naturally represents estimated spatial coordinates ($z_i = \widehat{s}_i$), and could be naturally adopted in spatial reconstruction and imputation tasks.
However, when latent representations are intended for other applications, such as domain identification, trajectory inference, and gene expression imputation, $d=2$ is not recommended, as the intrinsic dimension of gene expression manifolds typically exceeds two. Deliberately constraining to two dimensions may over-compress features, causing information loss and reducing effectiveness for these tasks.
}
Tuning for the latent dimension could also require a changed set of training hyperparameters, which should be handled with caution.

The whole algorithm is summarized in \Cref{alg:dsr}.
In the ideal case where the algorithm is capable of fully capturing the distribution, it recovers the real spatial information matrix \textit{invariant under translation, scaling, and rotation}.
This invariance is sufficient in practice because the relative spatial relationships between cells are preserved, which is the key information needed for most spatial transcriptomics analyses. The absolute coordinates are often arbitrary and less important than the relative positions and distances between cells in the tissue context.
We will describe how such similarity can be evaluated in the experiment section.

\begin{algorithm}[!t]
\caption{Spatial Reconstruction with \dpVAE}
\label{alg:dsr}
\begin{algorithmic}[1]
    \REQUIRE Training datasets $X_{\rm tr}, S_{\rm tr}$, test dataset $X_{\rm te}$.
    \STATE Train \dpVAE model via maximizing the distance-preserving ELBO \eqref{eq:objective}.
    \STATE Compute latent encoding $\widehat{Z}_{\rm te}$ from the trained encoder using $X_{\rm te}$, and obtain its pairwise distance matrix $\widehat{D}_{\rm te}$.
    \STATE Solve optimization problem in \eqref{eq:opt}-\eqref{eq:const} to obtain coordinate matrix $\widehat{S}_{\rm te}$.
    \RETURN Coordinate matrix $\widehat{S}_{\rm te}$.
\end{algorithmic}
\end{algorithm}

\subsection{Theoretical Property}
\label{sec:theory}

We briefly comment on our findings on the theoretical connection between the proposed distance-preserving loss and the distortion  \cite{chennuru2018measures} or bi-Lipshitz condition\cite{mahabadi2018nonlinear, verine2023expressivity}. Specifically, the enforcement result can be viewed as a probably approximately correct (PAC) learning version of the two, summarized in the following theorem.
\begin{theorem}
    \label{thm}
    For any $\epsilon > 0$, there exist some constants $\lambda > 0$ and
    \begin{equation} \label{eq:prop-L}
        L \lesssim \mathcal{O} \left( \frac{\mathbb{E}\left[ \mathcal{L}_{\mathrm{DP}} \right] }{\lambda \epsilon} \right).
    \end{equation}
    such that the following condition holds with probability greater than $1 - \epsilon$:
    \begin{equation} \label{eq:bi-lip}
        \lambda \| s - s' \| \leq \| z - z' \|  \leq L \cdot \lambda \| s - s' \|,
    \end{equation}
    where the probability arises from the following generation process
    \begin{equation} \label{eq:gen-process}
        (x, s), (x', s') \overset{iid}{\sim} p,
        \ z \sim f_\theta(y),
        \ z' \sim f_\theta(y').
    \end{equation}
\end{theorem}

In \Cref{thm}, expression \eqref{eq:bi-lip} resembles the distortion or bi-Lipschitz condition defined in prior works, except that it only holds probabilistically.
This probability is incurred by the underlying data generation process and the learned VAE generative model.
The distortion constant or bi-Lipschitz constant $L$ is proportional to the expected value of the distance-preserving loss $\mathcal{L}_{\mathrm{DP}}$, which in the ideal case (\textit{i.e.} perfect generalization) scales with the value of its training error. This means that the encoder network becomes ``smoother'' as training progresses, and the distance-preserving loss is what regulates such behavior of this generative model.
We note that this result summarizes and can be potentially generalized to characterize the theoretical properties of the generative networks regulated with distance-preserving loss, as it is not unique to the \dpVAE framework we are using in this paper.
We provide the exact bound to Theorem~\ref{thm} in the appendix. 

\section{Experiments}
\label{sec:experiment}

\begin{figure}
    \centering
    \includegraphics[width=0.7\linewidth]{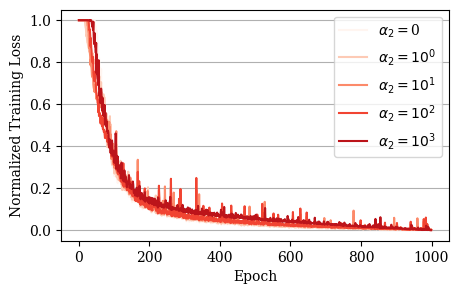}
    \caption{Training stability for different choices of regularization coefficients for the distance-preserving loss over $1 \times 10^3$ epochs of training, averaged over 27 datasets.}
    \label{fig:training-robustness}
\end{figure}

\begin{figure*}
    \centering
    \includegraphics[width=0.9\linewidth]{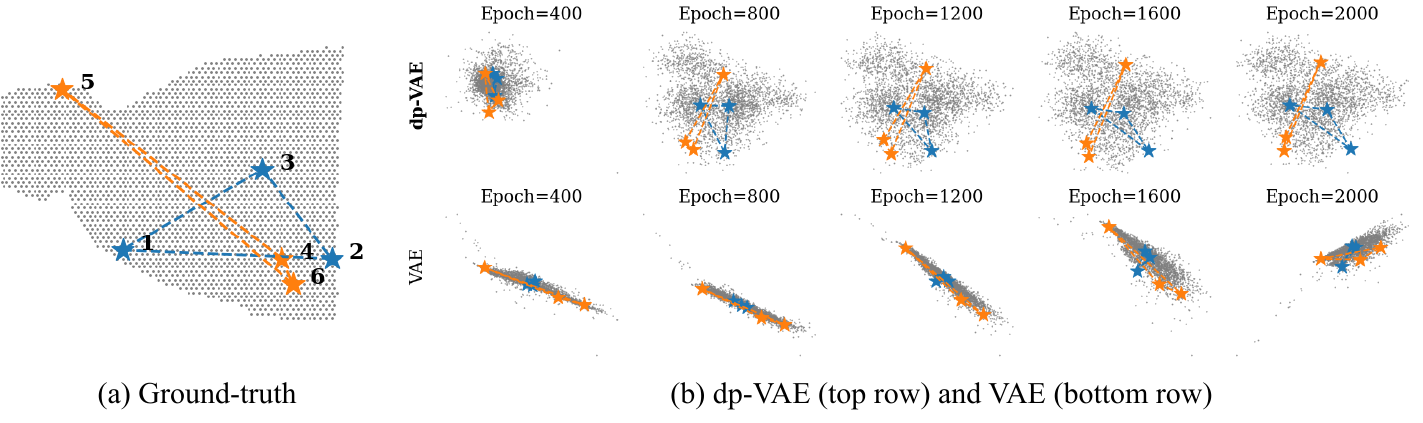}
    \vspace{-1ex}
    \caption{Latent space visualization over different training epochs.
    Each dot represents the 2-dimensional embedding of a cell.
    The blue and orange stars represent two randomly selected triplets of spatial locations.}
    \label{fig:traning-progress}
\end{figure*}

\begin{figure}
    \centering
    \includegraphics[width=0.7\linewidth]{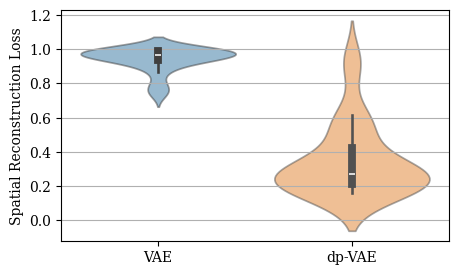}
    \caption{Comparions of out-of-sample spatial reconstruction error between VAE and \dpVAE, evaluated on 27 datasets.}
    \label{fig:baseline-comparison}
\end{figure}

\begin{figure*}
    \centering
    \includegraphics[width=0.7\linewidth]{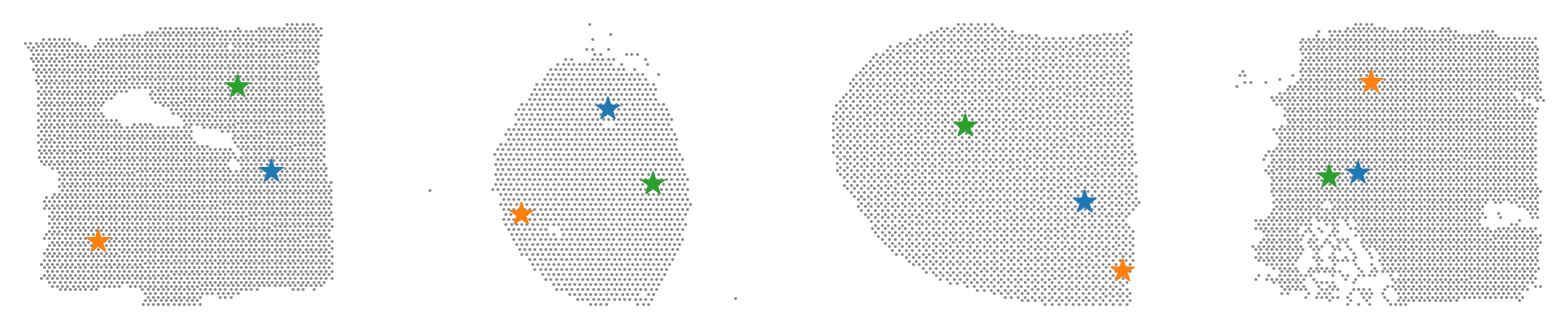}
    \includegraphics[width=0.7\linewidth]{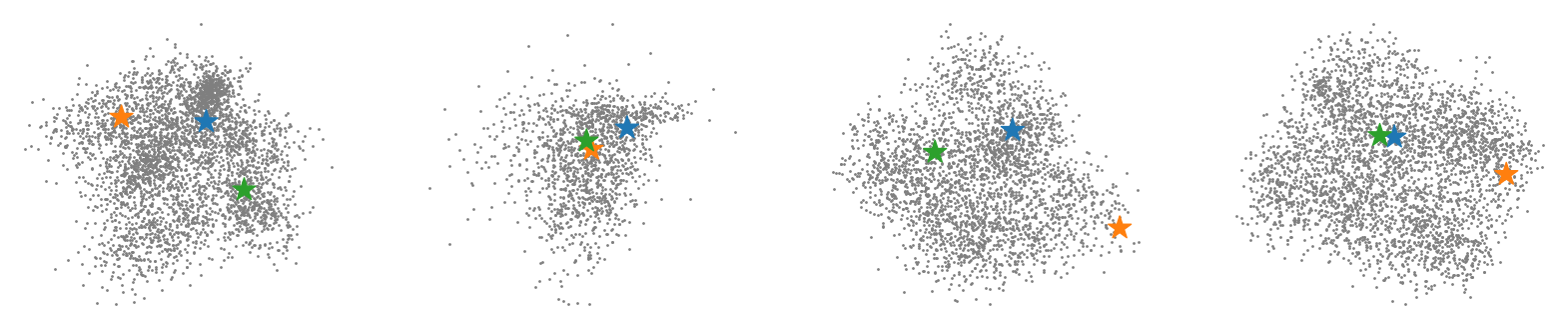}
    \caption{Out of sample (leave-three-out) spatial location imputation. Stars of different colors represent a randomly selected fixed triplet. The top row is the ground truth spatial context, and the bottom row is the reconstructed spatial context.}
    \label{fig:OOS}
\end{figure*}

\begin{figure}[t]
    \centering
    \includegraphics[width=0.8\linewidth]{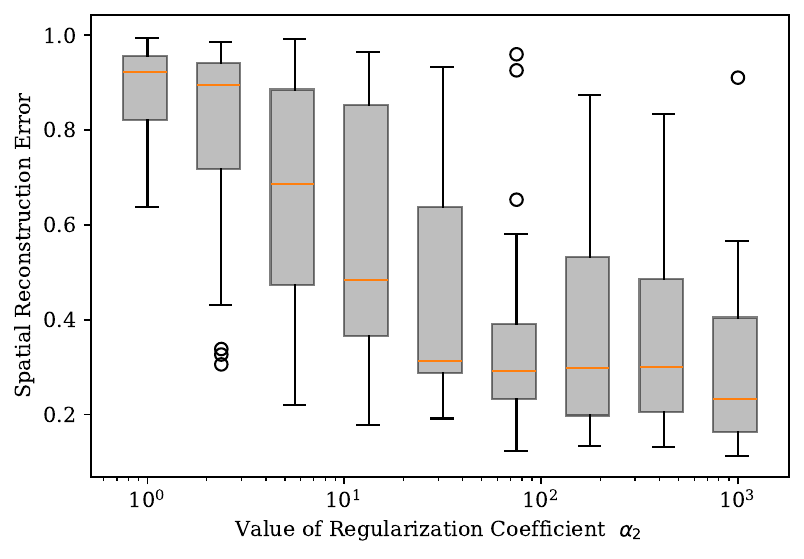}
    \caption{Sensitivity analysis of spatial reconstruction error with different regularization coefficient choices, evaluated over 27 datasets.
    }
    \label{fig:sensitivity-analysis}
\end{figure}

In this section, we conduct experiments to validate the effectiveness of our method and shed light on some of our method's practical implications, including training stability and in-sample behavior of \dpVAE (\Cref{sec:train}), its comparison against vanilla VAE in out-of-sample evaluation (\Cref{sec:oos}), and out-of-distribution inference applications, \textit{i.e.,} transfer learning (\Cref{sec:transfer}). We use a total of 27 processed Visium Spatial Gene Expression data from the 10x Genomics database provided by Scanpy \cite{wolf2018scanpy}.
We only select the top $100$ most variable genes for all datasets to ensure model stability and learning feasibility \footnote{Most datasets only contain roughly around three thousand entries, making the problem high-dimensional and numerically unstable.}
The names of the datasets are the initials of the datasets, whose full names and correspondence can be found in the appendix.
Before we proceed, we first describe our experiment setups in the next subsection.

\subsection{Experiment Setups}
\label{sec:exp-setup}

\vspace{1ex}
{\noindent \bf Model description.}
Since the VAE model is a special case of \dpVAE when the distance-preserving regularization coefficient is set to zero. 
Therefore, we will only describe our setup for the \dpVAE model.
The \dpVAE model consists of an encoder and a decoder network, each comprising a one-layer fully connected neural network with {\color{black} $128$} hidden layers and a latent dimension of $2$; therefore, the procedure of solving for an optimization problem can be bypassed by directly using the latent representations as the estimated spatial coordinates.
The prior, posterior, and likelihood distributions are all specified as Gaussian distributions.
We apply no mask between any pair of entries.
The regularization coefficients $\alpha_1, \beta$ for the ELBO objective are set to one throughout all experiments; therefore, for simplicity, we will only refer to $\alpha_2$ as the regularization coefficient in our experiment section.
The learning rate is set to $5 \times 10^{-4}$ or $10^{-4}$ depending on the experiment, and the total number of training epochs is set to $2000$.
All experiments are run on an NVIDIA GeForce RTX 4070Ti GPU.

\begin{figure}
    \centering
    \includegraphics[width=0.8\linewidth]{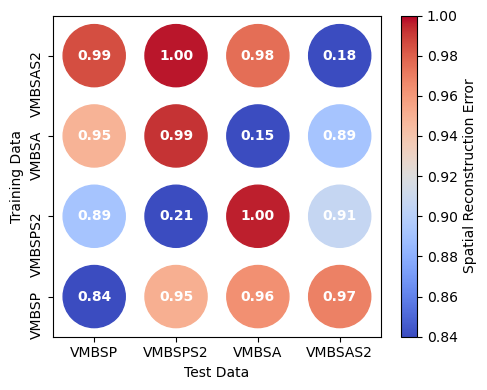}
    \caption{The matrix of out-of-distribution spatial reconstruction error incurred by \dpVAE. The model uses each column of data as the training data to infer four test datasets labeled on the rows.}
    \label{fig:transfer-learning}
\end{figure}

\vspace{1ex}
{\noindent \bf Evaluation protocals.}
We evaluate our model based on the Procrustes distance \cite{goodall1991procrustes, gower1975generalized, krzanowski2000principles} between the reconstructed and the ground-truth spatial information matrix. Formally, for two datasets denoted by matrices $\textbf{D}_1, \textbf{D}_2 \in \mathbb{R}^{n \times m}$, the procedure of calculating the Proscrutes distance between $\textbf{D}_1$ and $\textbf{D}_2$ can be defined as the procedure of finding the optimal translation, rotation, and scaling transformation from $\mathbb{R}^m$ to $\mathbb{R}^m$, such that the Frobenius norm between the two matrices is minimized. Formally, it can be described as
\begin{align*}
    \min_{\textbf{A}, b, c} \quad & \| c \cdot \textbf{A} \textbf{D}_1 + b - \textbf{D}_2 \|_{F}.
\end{align*}
Here $\textbf{A} \in \mathbb{R}^{m, m}$ is an orthogonal matrix, $b \in \mathbb{R}^m$, and $c$ is a scalar.
We will refer to it as the \textit{spatial reconstruction error} in the following sections when clear from context.
{\color{black} Across all of our experiments, the errors are normalized within $[0, 1]$ for brevity of comparison.}

\subsection{In-Sample Training Behaviors}
\label{sec:train}

First, we evaluate the training stability of \dpVAE under various choices of the regularization coefficient. This analysis is crucial because the proposed modification of the original ELBO objective may introduce uncertainty about whether the new objective ensures convergence or risks numerical instability.

Figure~\ref{fig:training-robustness} shows the normalized training loss curve of the \dpVAE model with different choices of the regularization coefficient $\alpha_2$ averaged over 27 datasets used as training data. Our results show that increasing $\alpha_2$ may introduce small abrupt oscillations in the loss curve; globally, the loss manages to converge at the same point for all choices of $\alpha_2$. This means that different choices of $\alpha_2$ do not significantly impact the convergence speed and stability of the training, making its training behavior almost identical to the standard VAE model. In practice, \dpVAE can be tuned and trained very similarly to a standard VAE. However, it is noted that using large learning rates may amplify the oscillation in the training, making it numerically unstable.

Figure~\ref{fig:traning-progress} reports the reconstructed spatial coordinates at different stages of training (every $400$ epochs) of \dpVAE compared with VAE on the Mouse Brain Sagittal Anterior Section 2 dataset. We randomly pick two sets of three anchor points to help visualize how the distance-preserving property is gradually enforced. It can be seen that the \dpVAE model quickly converges within $800$-th to accurately capture the global shape of the spatial context in the training data.
At more refined resolutions, the relative locations of the anchor points have been accurately recovered by \dpVAE, where the reconstructed triangular shapes are almost identical to the ground-truth data. On the other hand, the reconstructed spatial context by the vanilla VAE model does not converge correctly to the training data.
This implies that the distance-preserving loss in the \dpVAE model enforces disentangled latent representations from gene expressions, and captures meaningful biological signals helpful for the task of spatial reconstruction. 

\begin{figure}
    \centering
    \includegraphics[width=1\linewidth]{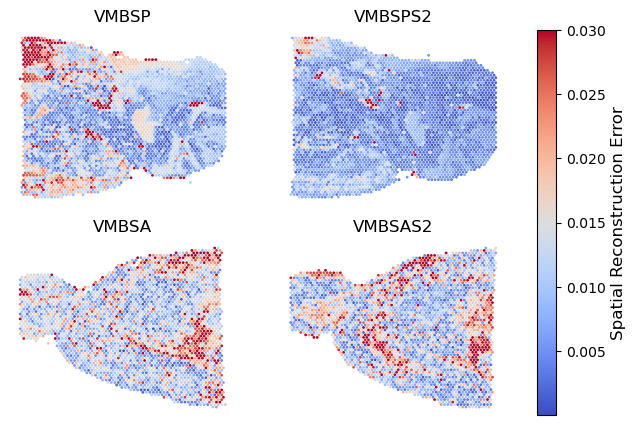}
    \caption{Heatmap of spatial reconstruction error of four Visium mouse brain datasets obtained from \dpVAE, where the training data is selected to be only VMBSP.}
    \label{fig:spatial-reconstruction}
\end{figure}

\subsection{Out-of-Sample Evaluation}
\label{sec:oos}

To evaluate the out-of-sample generalization ability and verify that the \dpVAE does not simply memorize the training data, we carry out comparative studies comparing VAE and \dpVAE, as well as different choices of the regularization coefficient $\alpha_2$ on out-of-sample spatial reconstruction evaluation tasks. We obtain our training and test data by dividing each one of $27$ datasets by a ratio of $80\%-20\%$, yielding 27 training datasets paired with 27 test datasets.
{
\color{black}
We note that the comparison between \texttt{dp-VAE} and VAE primarily serves as an ablation study to assess the effectiveness of the proposed distance-preserving loss. This comparison is also fair, as, to the best of our knowledge, no existing model is designed to exploit spatial information $S$ without requiring it at the inference stage, which is a setting unique to spatial reconstruction tasks defined in this paper (\Cref{sec:problem-def}).
}

Figure~\ref{fig:baseline-comparison} reports the normalized spatial reconstruction error comparing the VAE and \dpVAE models on the 27 test datasets. It can be seen that the \dpVAE significantly outperforms vanilla VAE on spatial reconstruction tasks across almost all datasets, demonstrating the generalization capability of the distance-preserving regularizer. This result aligns with our previous theoretical analysis. Furthermore, Figure~\ref{fig:corr-X-S} indicates a statistical dependence between gene expression and its spatial context, suggesting that spatial context can be inferred from signals derived from gene expressions, even under mild extrapolation.

Figure~\ref{fig:OOS} follows Figure~\ref{fig:baseline-comparison} and reports the detailed reconstructed spatial context (bottom row) with the real spatial context (top row). The three colored spots are random anchor points chosen within the test dataset. It can be seen that the three imputed anchor points are placed in the correct relative position within the reconstructed spatial context. In the meantime, the relative locations and distances of the three points are also accurately represented. This confirms the nuanced rationale underlying the strong performance on the aggregated metric of spatial reconstruction error.

Figure~\ref{fig:sensitivity-analysis} presents the spatial reconstruction error averaged over 27 test datasets, evaluated under varying regularization coefficient values. As the regularization coefficient increases, the test (out-of-sample) spatial reconstruction error decreases logarithmically. This finding further validates that applying stronger distance-preserving regularization in \dpVAE enhances its ability for spatial reconstruction extrapolation. 

However, we note that using a larger regularization coefficient is not always optimal in practice. With larger networks, excessively strong regularization can increase the risk of overfitting the training data, ultimately impairing extrapolation performance. Furthermore, as previously mentioned, a high regularization coefficient can lead to instability during training, potentially causing the learning process to fail. Therefore, selecting an appropriate regularization coefficient requires careful consideration of various practical factors.

\subsection{Out-of-Distribution Evaluation and Inference}
\label{sec:transfer}

In previous sections, we have been evaluating the performance of our methods, where the training and test data are split from the same dataset.
We further apply our method to a more challenging scenario, where the train and test data come from separate datasets, and biological structure may be potentially different in the two populations.
We experiment on four mouse brain tissue spatial transcriptomics datasets, creating $4 \times 4 = 16$ pairs of training and test datasets. We aim to apply the proposed \dpVAE to extract insights about the transfer learning capability between these four datasets. 

\Cref{fig:transfer-learning} visualizes the normalized spatial reconstruction error matrix, where the columns are the test data where the error is calculated, and the rows are the training data where the \dpVAE model is trained. It can be seen that the train-test pairs VMBSPS2--VMBSP, VMBSPS2--VMBSP, and VMBSA--VMBSAS2 incur relatively small error scales compared to other train-test pairs. This may indicate that the joint distribution of gene expression and spatial context between these sets of tissues is similar, and this shared information may be used to infer one another.

Figure~\ref{fig:spatial-reconstruction} shows the detailed view of the spatial reconstruction error at the coordinate level of four test datasets using VMBSP as the training data, corresponding to the bottom row of Figure~\ref{fig:transfer-learning}. 
The results reveal distinct patterns of spatial reconstruction error between different tissue regions, suggesting that heterogeneous gene expression carries varying levels of biological signals closely linked to its spatial context. Further research is anticipated to illuminate the underlying causes of this heterogeneity.

{
\color{black}
Finally, we do want to note that in \Cref{fig:transfer-learning}, the off-diagonal elements of the matrix (out-of-distribution error) take values close to one, while the diagonal elements (in-sample error) remain relatively low and close to zero. This result shows that the out-of-distribution inference of our method has a large error relative to its in-sample inference performance.
We hypothesize that the large errors are primarily attributable to our zero-shot inference setting, where test data originates directly from different brain sectors than the training data without direct overlap.
This setting is particularly challenging because: ($i$) different brain regions exhibit distinct gene expression patterns and spatial organizations, and ($ii$) our model has no prior exposure to the target domain's characteristics during training. Such performance degradation is expected in cross-domain spatial genomics, as the learned spatial-expression mappings may not generalize across anatomically distinct regions.

Given the inherent challenges of zero-shot learning, one potential way to mitigate this limitation is through fine-tuning. In practical genomic studies, researchers often have access to a small amount of target-domain data (or data closely resembling it), which can be leveraged for fine-tuning. This strategy has been extensively tested in prior work and has consistently shown substantial error reductions. Therefore, it is reasonable to hypothesize that, with suitable fine-tuning, the out-of-distribution inference task in our setting could also be improved. In general, achieving strong zero-shot generalization remains difficult and constitutes an important direction for future research.
}

\section{Conclusion and Discussion}

In this paper, we propose a novel representation learning framework based on a variational autoencoder structure to enable tissue spatial reconstruction and imputation from gene expression data.
Note that our approach relies minimally on spatial transcriptomics in the sense that they are only used during the training stage as an auxiliary variable, which allows for diverse downstream applications that rely on having spatial context to be carried out when such information is unavailable during the inference stage, opening up the possibility of a broad spectrum of genomics-related studies that might be once unavailable on many datasets.
The technical novelty lies in the distance-preserving loss that we introduce to regulate the training process, which helps the model to effectively capture gene expression signals oriented towards describing its spatial context.
Spatial reconstruction can be achieved by solving an optimization problem or automatically achieved when the latent space dimension is set to be the dimension of the spatial context.
We also point out the theoretical gap between the distance-preserving loss or bi-Lipschitz condition in the context of generative modeling. 
Finally, we study the training robustness, out-of-sample inference, and transfer learning performance of our proposed method on 27 spatial transcriptomics datasets to evaluate and demonstrate its empirical effectiveness.

A potential future research direction is to more extensively evaluate the method in more challenging transfer learning settings, such as examining the performance of the method in inferring the spatial context of gene expressions that are at a larger scale or with higher spatial resolution, studying the effect of different masking procedures \cite{chen2022local}, performance in different downstream analysis when integrated with other algorithms, or examine the performance of the method when using reference datasets with higher dissimilarity to the inference datasets. This comparative study might need to incorporate more biological domain expert insights and is more resource-demanding. Nevertheless, we believe that this general framework has strong potential to be modified to various tailored settings.  

\bibliographystyle{ieeetr}
\bibliography{ref}

@String{Computing = "Computing" }

@String{Computer = "{IEEE} Computer" }

@String{Psychometrika = "Psychometrika" }

@String{Springer = "Springer-Verlag" }

@misc{vae,
      title={Auto-Encoding Variational Bayes}, 
      author={Diederik P Kingma and Max Welling},
      year={2022},
      eprint={1312.6114},
      archivePrefix={arXiv},
      primaryClass={stat.ML},
      url={https://arxiv.org/abs/1312.6114}, 
}

@article{sohn2015learning,
  title={Learning structured output representation using deep conditional generative models},
  author={Sohn, Kihyuk and Lee, Honglak and Yan, Xinchen},
  journal={Advances in neural information processing systems},
  volume={28},
  year={2015}
}

@inproceedings{wang2019spatial,
  title={Spatial Variational Auto-Encoding via Matrix-Variate Normal Distributions},
  author={Wang, Zhengyang and Yuan, Hao and Ji, Shuiwang},
  booktitle={Proceedings of the 2019 SIAM International Conference on Data Mining},
  pages={648--656},
  year={2019},
  organization={SIAM}
}

@article{semenova2022priorvae,
  title={PriorVAE: encoding spatial priors with variational autoencoders for small-area estimation},
  author={Semenova, Elizaveta and Xu, Yidan and Howes, Adam and Rashid, Theo and Bhatt, Samir and Mishra, Swapnil and Flaxman, Seth},
  journal={Journal of the Royal Society Interface},
  volume={19},
  number={191},
  pages={20220094},
  year={2022},
  publisher={The Royal Society}
}

@article{chen2023spatial,
  title={Spatial transcriptomic technologies},
  author={Chen, Tsai-Ying and You, Li and Hardillo, Jose Angelito U and Chien, Miao-Ping},
  journal={Cells},
  volume={12},
  number={16},
  pages={2042},
  year={2023},
  publisher={MDPI}
}

@article{yu2022spatial,
  title={Spatial transcriptomics technology in cancer research},
  author={Yu, Qichao and Jiang, Miaomiao and Wu, Liang},
  journal={Frontiers in Oncology},
  volume={12},
  pages={1019111},
  year={2022},
  publisher={Frontiers}
}

@article{dries2021advances,
  title={Advances in spatial transcriptomic data analysis},
  author={Dries, Ruben and Chen, Jiaji and Del Rossi, Natalie and Khan, Mohammed Muzamil and Sistig, Adriana and Yuan, Guo-Cheng},
  journal={Genome research},
  volume={31},
  number={10},
  pages={1706--1718},
  year={2021},
  publisher={Cold Spring Harbor Lab}
}

@article{zeng2022statistical,
  title={Statistical and machine learning methods for spatially resolved transcriptomics data analysis},
  author={Zeng, Zexian and Li, Yawei and Li, Yiming and Luo, Yuan},
  journal={Genome biology},
  volume={23},
  number={1},
  pages={83},
  year={2022},
  publisher={Springer}
}

@article{hu2021spagcn,
  title={SpaGCN: Integrating gene expression, spatial location and histology to identify spatial domains and spatially variable genes by graph convolutional network},
  author={Hu, Jian and Li, Xiangjie and Coleman, Kyle and Schroeder, Amelia and Ma, Nan and Irwin, David J and Lee, Edward B and Shinohara, Russell T and Li, Mingyao},
  journal={Nature methods},
  volume={18},
  number={11},
  pages={1342--1351},
  year={2021},
  publisher={Nature Publishing Group US New York}
}

@article{zhu2020integrative,
  title={Integrative spatial single-cell analysis with graph-based feature learning},
  author={Zhu, Junjie and Sabatti, Chiara},
  journal={Biorxiv},
  pages={2020--08},
  year={2020},
  publisher={Cold Spring Harbor Laboratory}
}

@article{lopez2019joint,
  title={A joint model of unpaired data from scRNA-seq and spatial transcriptomics for imputing missing gene expression measurements},
  author={Lopez, Romain and Nazaret, Achille and Langevin, Maxime and Samaran, Jules and Regier, Jeffrey and Jordan, Michael I and Yosef, Nir},
  journal={arXiv preprint arXiv:1905.02269},
  year={2019}
}

@article{lopez2018deep,
  title={Deep generative modeling for single-cell transcriptomics},
  author={Lopez, Romain and Regier, Jeffrey and Cole, Michael B and Jordan, Michael I and Yosef, Nir},
  journal={Nature methods},
  volume={15},
  number={12},
  pages={1053--1058},
  year={2018},
  publisher={Nature Publishing Group US New York}
}

@article{li2007beyond,
  title={Beyond Moran's I: testing for spatial dependence based on the spatial autoregressive model},
  author={Li, Hongfei and Calder, Catherine A and Cressie, Noel},
  journal={Geographical analysis},
  volume={39},
  number={4},
  pages={357--375},
  year={2007},
  publisher={Wiley Online Library}
}

@article{kingma2016improved,
  title={Improved variational inference with inverse autoregressive flow},
  author={Kingma, Durk P and Salimans, Tim and Jozefowicz, Rafal and Chen, Xi and Sutskever, Ilya and Welling, Max},
  journal={Advances in neural information processing systems},
  volume={29},
  year={2016}
}

@article{sonderby2016ladder,
  title={Ladder variational autoencoders},
  author={S{\o}nderby, Casper Kaae and Raiko, Tapani and Maal{\o}e, Lars and S{\o}nderby, S{\o}ren Kaae and Winther, Ole},
  journal={Advances in neural information processing systems},
  volume={29},
  year={2016}
}

@misc{luo2022understanding,
      title={Understanding Diffusion Models: A Unified Perspective}, 
      author={Calvin Luo},
      year={2022},
      eprint={2208.11970},
      archivePrefix={arXiv},
      primaryClass={cs.LG}
}

@article{beshkov2022isometric,
  title={Isometric Representations in Neural Networks Improve Robustness},
  author={Beshkov, Kosio and Verhellen, Jonas and Lepper{\o}d, Mikkel Elle},
  journal={arXiv preprint arXiv:2211.01236},
  year={2022}
}

@article{bepler2019explicitly,
  title={Explicitly disentangling image content from translation and rotation with spatial-VAE},
  author={Bepler, Tristan and Zhong, Ellen and Kelley, Kotaro and Brignole, Edward and Berger, Bonnie},
  journal={Advances in Neural Information Processing Systems},
  volume={32},
  year={2019}
}

@proceedings{guo2021deep,
  title={Deep generative models for spatial networks},
  author={Guo, Xiaojie and Du, Yuanqi and Zhao, Liang},
  booktitle={Proceedings of the 27th ACM SIGKDD Conference on Knowledge Discovery \& Data Mining},
  pages={505--515},
  year={2021}
}

@article{edsgard2018identification,
  title={Identification of spatial expression trends in single-cell gene expression data},
  author={Edsg{\"a}rd, Daniel and Johnsson, Per and Sandberg, Rickard},
  journal={Nature methods},
  volume={15},
  number={5},
  pages={339--342},
  year={2018},
  publisher={Nature Publishing Group US New York}
}

@article{sun2020statistical,
  title={Statistical analysis of spatial expression patterns for spatially resolved transcriptomic studies},
  author={Sun, Shiquan and Zhu, Jiaqiang and Zhou, Xiang},
  journal={Nature methods},
  volume={17},
  number={2},
  pages={193--200},
  year={2020},
  publisher={Nature Publishing Group US New York}
}

@article{zhu2021spark,
  title={SPARK-X: non-parametric modeling enables scalable and robust detection of spatial expression patterns for large spatial transcriptomic studies},
  author={Zhu, Jiaqiang and Sun, Shiquan and Zhou, Xiang},
  journal={Genome biology},
  volume={22},
  number={1},
  pages={184},
  year={2021},
  publisher={Springer}
}

@article{svensson2018spatialde,
  title={SpatialDE: identification of spatially variable genes},
  author={Svensson, Valentine and Teichmann, Sarah A and Stegle, Oliver},
  journal={Nature methods},
  volume={15},
  number={5},
  pages={343--346},
  year={2018},
  publisher={Nature Publishing Group}
}

@inproceedings{hadsell2006dimensionality,
  title={Dimensionality reduction by learning an invariant mapping},
  author={Hadsell, Raia and Chopra, Sumit and LeCun, Yann},
  booktitle={2006 IEEE computer society conference on computer vision and pattern recognition (CVPR'06)},
  volume={2},
  pages={1735--1742},
  year={2006},
  organization={IEEE}
}

@article{mao2019metric,
  title={Metric learning for adversarial robustness},
  author={Mao, Chengzhi and Zhong, Ziyuan and Yang, Junfeng and Vondrick, Carl and Ray, Baishakhi},
  journal={Advances in neural information processing systems},
  volume={32},
  year={2019}
}

@article{zhou2023integrating,
  title={Integrating spatial transcriptomics data across different conditions, technologies and developmental stages},
  author={Zhou, Xiang and Dong, Kangning and Zhang, Shihua},
  journal={Nature Computational Science},
  volume={3},
  number={10},
  pages={894--906},
  year={2023},
  publisher={Nature Publishing Group US New York}
}

@article{li2024high,
  title={High-density generation of spatial transcriptomics with STAGE},
  author={Li, Shang and Gai, Kuo and Dong, Kangning and Zhang, Yiyang and Zhang, Shihua},
  journal={Nucleic Acids Research},
  pages={gkae294},
  year={2024},
  publisher={Oxford University Press}
}

@article{guo2023spiral,
  title={SPIRAL: integrating and aligning spatially resolved transcriptomics data across different experiments, conditions, and technologies},
  author={Guo, Tiantian and Yuan, Zhiyuan and Pan, Yan and Wang, Jiakang and Chen, Fengling and Zhang, Michael Q and Li, Xiangyu},
  journal={Genome Biology},
  volume={24},
  number={1},
  pages={241},
  year={2023},
  publisher={Springer}
}

@article{wan2023integrating,
  title={Integrating spatial and single-cell transcriptomics data using deep generative models with SpatialScope},
  author={Wan, Xiaomeng and Xiao, Jiashun and Tam, Sindy Sing Ting and Cai, Mingxuan and Sugimura, Ryohichi and Wang, Yang and Wan, Xiang and Lin, Zhixiang and Wu, Angela Ruohao and Yang, Can},
  journal={Nature Communications},
  volume={14},
  number={1},
  pages={7848},
  year={2023},
  publisher={Nature Publishing Group UK London}
}

@article{long2023spatially,
  title={Spatially informed clustering, integration, and deconvolution of spatial transcriptomics with GraphST},
  author={Long, Yahui and Ang, Kok Siong and Li, Mengwei and Chong, Kian Long Kelvin and Sethi, Raman and Zhong, Chengwei and Xu, Hang and Ong, Zhiwei and Sachaphibulkij, Karishma and Chen, Ao and others},
  journal={Nature Communications},
  volume={14},
  number={1},
  pages={1155},
  year={2023},
  publisher={Nature Publishing Group UK London}
}

@article{yang2022sc,
  title={SC-MEB: spatial clustering with hidden Markov random field using empirical Bayes},
  author={Yang, Yi and Shi, Xingjie and Liu, Wei and Zhou, Qiuzhong and Chan Lau, Mai and Chun Tatt Lim, Jeffrey and Sun, Lei and Ng, Cedric Chuan Young and Yeong, Joe and Liu, Jin},
  journal={Briefings in bioinformatics},
  volume={23},
  number={1},
  pages={bbab466},
  year={2022},
  publisher={Oxford University Press}
}

@article{dong2022deciphering,
  title={Deciphering spatial domains from spatially resolved transcriptomics with an adaptive graph attention auto-encoder},
  author={Dong, Kangning and Zhang, Shihua},
  journal={Nature communications},
  volume={13},
  number={1},
  pages={1739},
  year={2022},
  publisher={Nature Publishing Group UK London}
}

@article{xu2024unsupervised,
  title={Unsupervised spatially embedded deep representation of spatial transcriptomics},
  author={Xu, Hang and Fu, Huazhu and Long, Yahui and Ang, Kok Siong and Sethi, Raman and Chong, Kelvin and Li, Mengwei and Uddamvathanak, Rom and Lee, Hong Kai and Ling, Jingjing and others},
  journal={Genome Medicine},
  volume={16},
  number={1},
  pages={12},
  year={2024},
  publisher={Springer}
}

@article{dries2021giotto,
  title={Giotto: a toolbox for integrative analysis and visualization of spatial expression data},
  author={Dries, Ruben and Zhu, Qian and Dong, Rui and Eng, Chee-Huat Linus and Li, Huipeng and Liu, Kan and Fu, Yuntian and Zhao, Tianxiao and Sarkar, Arpan and Bao, Feng and others},
  journal={Genome biology},
  volume={22},
  pages={1--31},
  year={2021},
  publisher={Springer}
}

@article{du2020model,
  title={Joint trajectory inference for single-cell genomics using deep learning with a mixture prior},
  author={Du, Jin-Hong and Chen, Tianyu and Gao, Ming and Wang, Jingshu},
  journal={Proceedings of the National Academy of Sciences},
  volume={121},
  number={37},
  pages={e2316256121},
  year={2024},
  publisher={National Academy of Sciences}
}

@article{du2022robust,
  title={Robust probabilistic modeling for single-cell multimodal mosaic integration and imputation via scVAEIT},
  author={Du, Jin-Hong and Cai, Zhanrui and Roeder, Kathryn},
  journal={Proceedings of the National Academy of Sciences},
  volume={119},
  number={49},
  pages={e2214414119},
  year={2022},
  publisher={National Acad Sciences}
}

@inproceedings{wu2023counterfactual,
  title={Counterfactual generative models for time-varying treatments},
  author={Wu, Shenghao and Zhou, Wenbin and Chen, Minshuo and Zhu, Shixiang},
  booktitle={Proceedings of the 30th ACM SIGKDD Conference on Knowledge Discovery and Data Mining},
  pages={3402--3413},
  year={2024}
}

@article{liu2023comprehensive,
  title={A comprehensive overview of graph neural network-based approaches to clustering for spatial transcriptomics T. Liu et al. Overview of Spatial Transcriptomics’ Spatial Clutering},
  author={Liu, Teng and Fang, Zhao-Yu and Zhang, Zongbo and Yu, Yongxiang and Li, Min and Yin, Ming-Zhu},
  journal={Computational and Structural Biotechnology Journal},
  year={2023},
  publisher={Elsevier}
}

@article{zhang2022graph,
  title={Graph-based autoencoder integrates spatial transcriptomics with chromatin images and identifies joint biomarkers for Alzheimer’s disease},
  author={Zhang, Xinyi and Wang, Xiao and Shivashankar, GV and Uhler, Caroline},
  journal={Nature Communications},
  volume={13},
  number={1},
  pages={7480},
  year={2022},
  publisher={Nature Publishing Group UK London}
}

@article{hu2024spatially,
  title={Spatially contrastive variational autoencoder for deciphering tissue heterogeneity from spatially resolved transcriptomics},
  author={Hu, Yaofeng and Xiao, Kai and Yang, Hengyu and Liu, Xiaoping and Zhang, Chuanchao and Shi, Qianqian},
  journal={Briefings in Bioinformatics},
  volume={25},
  number={2},
  pages={bbae016},
  year={2024},
  publisher={Oxford University Press}
}

@inproceedings{yonghyeon2021regularized,
  title={Regularized autoencoders for isometric representation learning},
  author={Yonghyeon, LEE and Yoon, Sangwoong and Son, Minjun and Park, Frank C},
  booktitle={International Conference on Learning Representations},
  year={2021}
}

@article{lee2021neighborhood,
  title={Neighborhood reconstructing autoencoders},
  author={Lee, Yonghyeon and Kwon, Hyeokjun and Park, Frank},
  journal={Advances in Neural Information Processing Systems},
  volume={34},
  pages={536--546},
  year={2021}
}

@inproceedings{mahabadi2018nonlinear,
  title={Nonlinear dimension reduction via outer bi-lipschitz extensions},
  author={Mahabadi, Sepideh and Makarychev, Konstantin and Makarychev, Yury and Razenshteyn, Ilya},
  booktitle={Proceedings of the 50th Annual ACM SIGACT Symposium on Theory of Computing},
  pages={1088--1101},
  year={2018}
}

@inproceedings{chen2022local,
  title={Local distance preserving auto-encoders using continuous knn graphs},
  author={Chen, Nutan and van der Smagt, Patrick and Cseke, Botond},
  booktitle={Topological, Algebraic and Geometric Learning Workshops 2022},
  pages={55--66},
  year={2022},
  organization={PMLR}
}

@article{kim2020face,
  title={Face-to-Music Translation Using a Distance-Preserving Generative Adversarial Network with an Auxiliary Discriminator},
  author={Kim, Chelhwon and Port, Andrew and Patel, Mitesh},
  journal={arXiv preprint arXiv:2006.13469},
  year={2020}
}

@article{benaim2017one,
  title={One-sided unsupervised domain mapping},
  author={Benaim, Sagie and Wolf, Lior},
  journal={Advances in neural information processing systems},
  volume={30},
  year={2017}
}

@article{port2020earballs,
  title={Earballs: Neural transmodal translation},
  author={Port, Andrew and Kim, Chelhwon and Patel, Mitesh},
  journal={arXiv preprint arXiv:2005.13291},
  volume={1},
  year={2020}
}

@inproceedings{verine2023expressivity,
  title={On the expressivity of bi-Lipschitz normalizing flows},
  author={Verine, Alexandre and Negrevergne, Benjamin and Chevaleyre, Yann and Rossi, Fabrice},
  booktitle={Asian Conference on Machine Learning},
  pages={1054--1069},
  year={2023},
  organization={PMLR}
}

@article{chennuru2018measures,
  title={Measures of distortion for machine learning},
  author={Chennuru Vankadara, Leena and von Luxburg, Ulrike},
  journal={Advances in Neural Information Processing Systems},
  volume={31},
  year={2018}
}

@article{leeuw1977application,
  title={Application of convex analysis to multidimensional scaling},
  author={LEEUW, DE},
  journal={Recent developments in statistics},
  pages={133--145},
  year={1977},
  publisher={North-Holland}
}

@article{moriel2021novosparc,
  title={NovoSpaRc: flexible spatial reconstruction of single-cell gene expression with optimal transport},
  author={Moriel, Noa and Senel, Enes and Friedman, Nir and Rajewsky, Nikolaus and Karaiskos, Nikos and Nitzan, Mor},
  journal={Nature protocols},
  volume={16},
  number={9},
  pages={4177--4200},
  year={2021},
  publisher={Nature Publishing Group UK London}
}

@article{wolf2018scanpy,
  title={SCANPY: large-scale single-cell gene expression data analysis},
  author={Wolf, F Alexander and Angerer, Philipp and Theis, Fabian J},
  journal={Genome biology},
  volume={19},
  pages={1--5},
  year={2018},
  publisher={Springer}
}

@article{goodall1991procrustes,
  title={Procrustes methods in the statistical analysis of shape},
  author={Goodall, Colin},
  journal={Journal of the Royal Statistical Society: Series B (Methodological)},
  volume={53},
  number={2},
  pages={285--321},
  year={1991},
  publisher={Wiley Online Library}
}

@article{ma2022spatially,
  title={Spatially informed cell-type deconvolution for spatial transcriptomics},
  author={Ma, Ying and Zhou, Xiang},
  journal={Nature biotechnology},
  volume={40},
  number={9},
  pages={1349--1359},
  year={2022},
  publisher={Nature Publishing Group US New York}
}

@article{xu2021costa,
  title={CoSTA: unsupervised convolutional neural network learning for spatial transcriptomics analysis},
  author={Xu, Yang and McCord, Rachel Patton},
  journal={BMC bioinformatics},
  volume={22},
  number={1},
  pages={397},
  year={2021},
  publisher={Springer}
}

@article{wang2024scvsc,
  title={scVSC: Deep variational subspace clustering for single-cell transcriptome data},
  author={Wang, Zile and Wang, Haiyun and Zhao, Jianping and Xia, Junfeng and Zheng, Chunhou},
  journal={IEEE/ACM Transactions on Computational Biology and Bioinformatics},
  year={2024},
  publisher={IEEE}
}

@article{biancalani2021deep,
  title={Deep learning and alignment of spatially resolved single-cell transcriptomes with Tangram},
  author={Biancalani, Tommaso and Scalia, Gabriele and Buffoni, Lorenzo and Avasthi, Raghav and Lu, Ziqing and Sanger, Aman and Tokcan, Neriman and Vanderburg, Charles R and Segerstolpe, {\AA}sa and Zhang, Meng and others},
  journal={Nature methods},
  volume={18},
  number={11},
  pages={1352--1362},
  year={2021},
  publisher={Nature Publishing Group US New York}
}

@article{gower1975generalized,
  title={Generalized procrustes analysis},
  author={Gower, John C},
  journal={Psychometrika},
  volume={40},
  pages={33--51},
  year={1975},
  publisher={Springer}
}

@book{krzanowski2000principles,
  title={Principles of multivariate analysis},
  author={Krzanowski, Wojtek},
  volume={23},
  year={2000},
  publisher={OUP Oxford}
}

@article{higgins2017beta,
  title={beta-vae: Learning basic visual concepts with a constrained variational framework.},
  author={Higgins, Irina and Matthey, Loic and Pal, Arka and Burgess, Christopher P and Glorot, Xavier and Botvinick, Matthew M and Mohamed, Shakir and Lerchner, Alexander},
  journal={ICLR (Poster)},
  volume={3},
  year={2017}
}

@article{pan2009survey,
  title={A survey on transfer learning},
  author={Pan, Sinno Jialin and Yang, Qiang},
  journal={IEEE Transactions on knowledge and data engineering},
  volume={22},
  number={10},
  pages={1345--1359},
  year={2009},
  publisher={IEEE}
}

@article{saeed2018survey,
  title={A survey on multidimensional scaling},
  author={Saeed, Nasir and Nam, Haewoon and Haq, Mian Imtiaz Ul and Muhammad Saqib, Dost Bhatti},
  journal={ACM Computing Surveys (CSUR)},
  volume={51},
  number={3},
  pages={1--25},
  year={2018},
  publisher={ACM New York, NY, USA}
}

@ARTICLE{2020SciPy-NMeth,
  author  = {Virtanen, Pauli and Gommers, Ralf and Oliphant, Travis E. and
            Haberland, Matt and Reddy, Tyler and Cournapeau, David and
            Burovski, Evgeni and Peterson, Pearu and Weckesser, Warren and
            Bright, Jonathan and {van der Walt}, St{\'e}fan J. and
            Brett, Matthew and Wilson, Joshua and Millman, K. Jarrod and
            Mayorov, Nikolay and Nelson, Andrew R. J. and Jones, Eric and
            Kern, Robert and Larson, Eric and Carey, C J and
            Polat, {\.I}lhan and Feng, Yu and Moore, Eric W. and
            {VanderPlas}, Jake and Laxalde, Denis and Perktold, Josef and
            Cimrman, Robert and Henriksen, Ian and Quintero, E. A. and
            Harris, Charles R. and Archibald, Anne M. and
            Ribeiro, Ant{\^o}nio H. and Pedregosa, Fabian and
            {van Mulbregt}, Paul and {SciPy 1.0 Contributors}},
  title   = {{{SciPy} 1.0: Fundamental Algorithms for Scientific
            Computing in Python}},
  journal = {Nature Methods},
  year    = {2020},
  volume  = {17},
  pages   = {261--272},
  adsurl  = {https://rdcu.be/b08Wh},
  doi     = {10.1038/s41592-019-0686-2},
}

\appendices

\section{Omitted Proofs in \Cref{sec:theory}}
\label{app:proof}

In this section. we provide the detailed proof for \Cref{thm} in \Cref{sec:theory}.
We begin by describing a technical lemma that will be used in the main proof.

\begin{lemma} \label{lemma:dist-bound}
    For any given distribution defined on some space $\mathcal{X} \subseteq \mathbb{R}^n$, given arbitrary $ 0 <\delta < 1$, there exist some $0 < M_1 < M_2 < +\infty$, such that:
    \begin{equation}
        \mathbb{P}\left[ M_1 \leq \| x - x'\| \leq M_2 \right] \geq 1 - \delta,
    \end{equation}
    where $x, x'$ are drawn i.i.d. from this distribution.
\end{lemma}

\begin{proof}
    If $\mathcal{X}$ is a finite space or bounded space, then simply take:
    \[
    M_1 = \inf_{x, x' \in \mathcal{X}} \| x - x'\|, \
    M_2 = \sup_{x, x' \in \mathcal{X}} \| x - x'\|.
    \]
    If $\mathcal{X}$ is an infinite space or an unbounded space, such as $\mathcal{X} = \mathbb{R}^n$. In this case, since we know that the Euclidean space $\mathbb{R}^n$ can be exhausted by a nested sequence of compact balls centered at the origin:
    \[
    \mathcal{X} \subseteq \mathbb{R}^n = \bigcup_{r=1}^{+\infty} B(0, r),
    \]
    where $B(0, r) := \left\{ x \in \mathbb{R}^n : \| x \| \leq r \right\}$. Therefore, by the total probability property, we know that:
    \begin{equation}
        \mathbb{P}(\mathcal{X}) \leq \mathbb{P} \left[ \bigcup_{r=1}^{+\infty} B(0, r) \right] = 1.
    \end{equation}
    By the dominated convergence theorem:
    \begin{equation}
        \mathbb{P} \left[ \lim_{R \to +\infty} \bigcup_{r=1}^{R} B(0, r) \right] = \lim_{R \to +\infty} \mathbb{P} \left[ \bigcup_{r=1}^{R} B(0, r) \right].
    \end{equation}
    By definition, given any $0 < \delta < 1$, there exists some $R' > 0 $ such that:
    \[
    \mathbb{P}\left[ \bigcup_{r=1}^{R'} B(0, r) \right] \geq 1 - \delta,
    \]
    which implies:
    \[
    \mathbb{P}\left[ -2R' \leq \| x - x'\| \leq 2R' \right] \geq 1 - \delta.
    \]
    Therefore, we can set $M_1 = -2R'$ and $M_2 = 2R'$, and this concludes the proof.
\end{proof}

\begin{theorem}
    \label{thm:equiv}
    Denote $\delta, M_1, M_2$ be given by Lemma \ref{lemma:dist-bound} with $\mathcal{X}$ set to $\mathcal{S}$, let $\mathcal{L}_{\rm DP}$ be defined in \eqref{eq:dp}.
    Then, for any $\epsilon > 0$, there exist some constants $\lambda > 0$ and
    \begin{equation} 
        L \leq \frac{M_2}{M_1} + \frac{\mathbb{E}\left[ \mathcal{L}_{\mathrm{DP}} \right] }{\lambda M_1 (\epsilon - \delta)}.
    \end{equation}
    such that the following condition holds with probability greater than $1 - \epsilon$:
    \begin{equation} 
        \lambda \| s - s' \| \leq \| z - z' \|  \leq L \cdot \lambda \| s - s' \|,
    \end{equation}
    where the probability arises from the following generation process
    \begin{equation} 
        (x, s), (x', s') \overset{iid}{\sim} p,
        \ z \sim f_\theta(y),
        \ z' \sim f_\theta(y').
    \end{equation}
\end{theorem}

\begin{proof}
    Denote the following events in the probability space
    \begin{align*}
        \mathcal{A} & = \left\{  M_1 \leq \| s - s'\| \leq M_2 \right\}, \\
        \mathcal{B} & = \left\{ \lambda \| s - s'\|  - C \leq \| z - z'\| \leq \lambda \| s - s'\| + C  \right\} \\
        \mathcal{C} & = \left\{ \lambda \| s - s'\| \leq \| z - z'\| \leq L \cdot \lambda \| s - s'\| \right\},
    \end{align*}
    we will set
    $
    L = M_2 / M_1 + C / (\lambda M_1)
    $
    for now, which then, by the definition of the three events, there is
    \begin{equation}
        \label{eq:prop-2}
        \mathbb{P}(\mathcal{B} | \mathcal{A}) \leq \mathbb{P}(\mathcal{C} | \mathcal{A}).
    \end{equation}

    On the other hand, using Markov's inequality and by the definition of the distance preserving loss \eqref{eq:dp}, for any $C >0$, we have:
        \begin{align*}
            \frac{\mathbb{E} \left[ \mathcal{L}_{\textrm{DP}} \right]}{C}
            & \geq \mathbb{P}(\left| \| z - z'\| - \lambda \| s - s'\| \right| \geq C).
        \end{align*}
        Rearranging terms on both sides, we get
        \begin{align}
            & 1 - \frac{\mathbb{E} \left[ \mathcal{L}_{\textrm{DP}} \right]}{C} \notag\\
            & \leq \mathbb{P}(\left| \| z - z'\| - \lambda \| s - s'\| \right| \leq C) \notag \\
            & = \mathbb{P}\left( \lambda \| s - s'\|  - C \leq \| z - z'\| \leq \lambda \| s - s'\| + C \right) \notag \\
            & = \mathbb{P}(\mathcal{B}).
            \label{eq:ineq}
        \end{align}
    On the other hand, using the law of total probability, we can bound the probability of event $\mathcal{C}$ by:
    \begin{align*}
        \mathbb{P}(\mathcal{C}) & = \mathbb{P}(\mathcal{C} | \mathcal{A}) \mathbb{P}(\mathcal{A}) + \mathbb{P}(\mathcal{C} | \mathcal{A}^c) \mathbb{P}(\mathcal{A}^c) \\
        & \geq \mathbb{P}(\mathcal{B} | \mathcal{A}) (1 - \delta) \\
        & \geq \mathbb{P}(\mathcal{B}) - \delta \\
        & \geq 1 -{\mathbb{E}\left[\mathcal{L}_\textrm{DP}\right]} / {C} - \delta.
    \end{align*}
    The first line uses the law of total probability, the second line uses \eqref{eq:prop-2} and Lemma \ref{lemma:dist-bound}, and the third line uses the law of total probability and Lemma \ref{lemma:dist-bound}:
    \begin{align*}
        \mathbb{P}(\mathcal{B})
        & = \mathbb{P}(\mathcal{B}|\mathcal{A})\mathbb{P}(\mathcal{A}) + \mathbb{P}(\mathcal{B}|\mathcal{A}^c)\mathbb{P}(\mathcal{A}^c), \\
        & \leq \mathbb{P}(\mathcal{B}|\mathcal{A}) (1 - \delta) + \delta
    \end{align*}
    and the fourth line uses \eqref{eq:ineq}. Set $C = \mathbb{E}\left[\mathcal{L}_\textrm{DP}\right] / (\epsilon - \delta)$ and plug it back into the formula of $L$, then we can organize the previous expression and conclude when:
    \[
    L \leq \frac{M_2}{M_1} + \frac{\mathbb{E}\left[\mathcal{L}_\textrm{DP}\right]}{\lambda M_1 (\epsilon - \delta)},
    \]
    there is
    \[
    \mathbb{P}\left( \lambda \| s - s'\| \leq \| z - z'\| \leq L \cdot \lambda \| s - s'\| \right) \geq 1 - \epsilon.
    \]
    This proves what we desire.
\end{proof}

\section{Additional Related Works} \label{app:more-related-works}

\noindent\textbf{Spatial Data Modeling}
Autoencoders have been widely utilized in the analysis of spatial data due to their powerful representation learning capabilities. Traditional autoencoders (AE) have been explored in several studies \cite{li2024high,wan2023integrating}. For instance, SEDR \citep{xu2024unsupervised} employs a deep auto-encoder network to learn gene representations while incorporating a variational graph auto-encoder to embed spatial information simultaneously, highlighting the integration of spatial and genetic data. The combination of autoencoders with graph neural networks (GNNs) has also proven effective, as shown in works by \cite{long2023spatially,zhou2023integrating,guo2023spiral,liu2023comprehensive,zhang2022graph}, where such models enhance spatial data representation by leveraging graph structures. Additionally, advanced probabilistic models like the hidden Markov random field have been applied to spatial data, as demonstrated by Giotto \citep{dries2021giotto}, which identifies spatial domains using an HMRF model with a spatial neighbor prior. This approach has been extended in studies like \cite{yang2022sc}, further showcasing the potential of probabilistic methods in spatial data analysis.

\noindent\textbf{Spatial Variational Autoencoders}
Compared to the aforementioned approaches, variational autoencoders (VAEs) are recently attracting more attention in representation learning of single-cell data analysis \citep{du2022robust,du2020model}.
They are preferred over traditional autoencoders because they provide a structured, continuous, and interpretable latent space that facilitates robust generative capabilities and handles uncertainty by explicitly modeling distributions. This results in more meaningful data generation and easier manipulation of latent features compared to traditional autoencoders, which can then be used for data integration \citep{du2022robust}, trajectory inference \citep{du2020model}, etc.
Furthermore, the study of VAEs lays the groundwork for extending methodologies to other variational inference-based models, such as diffusion models and variational graph autoencoders.

A series of works have been conducted to study the interplay of VAEs with spatial information in the data, and different variants of spatial VAEs have been proposed. 
\cite{bepler2019explicitly} enables learning VAE models of images that separate latent variables encoding content from rotation and translation. This can be done by directly constraining them as part of the latent space to be learned. 
\citep{wang2019spatial} proposed to instead directly modify the latent space from the standard multivariate Gaussian distributions to low-ranked matrix-variate normal distributions to allow for better capturing ability of the spatial information present in image data. 
\cite{guo2021deep} considered using VAE to learn disentangled representations of spatial networks via careful designs of the model structure and latent variable factorization. They proposed a tractable optimization algorithm to carry out variational learning.
More recently, \cite{semenova2022priorvae} considered the setting of learning spatial Gaussian priors via a variational autoencoder.
\cite{hu2024spatially} combined graph variational autoencoder with contrastive learning.

However, a critical distinct difference to be noted is that while most existing works studying spatial VAE aim to capture the representation of the spatial information, our work aims to refine the representation learned by incorporating the spatial information given as context. This is a different but equally challenging task with broad applications in fields where spatial information is easily accessible, such as spatial transcriptomics data analysis.

\section{ELBO Derivation}
\label{app:elbo-derivation}

Below, we present the derivation of a lower bound for the log probability density function (PDF) of VAE, namely the evidence lower bound (ELBO).
To begin with, the marginal probability of data $x$ can be written as
\[
    \log f_{\theta}(x) = \log \int p_{\theta}(x, z )dz,
\]
where $z$ is a latent random variable.
This integral has no closed form and can usually be estimated by Monte Carlo integration with importance sampling, ie, 
\[
    \int p_{\theta}(x, z )dz = \EE_{z\sim q(\cdot |x )}\left[\frac{p_{\theta}(x, z)}{q(z|x)}\right].
\]
Here $q(z|x)$ is the proposed variational distribution, where we can draw a sample $z$ from this distribution given $x$ and $\ab$. 
Therefore, by Jensen's inequality, we can find the evidence lower bound (ELBO) of the conditional PDF:
\begin{align*}
    \log f_{\theta}(x) & = \log \EE_{z\sim q(\cdot|x )}\left[\frac{p_{\theta}(x, z)}{q(z|x)}\right] \\
    & \geq \EE_{z\sim q(\cdot|x)} \left[ \log \frac{p_{\theta}(x, z)}{q(z|x)}\right].
\end{align*}
Using Bayes rule, the ELBO can be equivalently expressed as:
\begin{align*}
& \EE_{z\sim q(\cdot|x)} \left[ \log \frac{p_{\theta}(x, z)}{q(z|x)}\right] \\
=& \EE_{z\sim q(\cdot|x)} \left[ \log \frac{p_{\theta}(x|z)p_{\theta}(z)}{q(z|x)}\right] \\ 
=& \EE_{z\sim q(\cdot|x)} \left[ \log \frac{p_{\theta}(z)}{q(z|x)}\right] + \EE_{z\sim q(\cdot|x)} \left[ \log p_{\theta}(x|z) \right] \\
=& - D_\text{KL}(q(z|x) || p_{\theta}(z)) + \EE_{z\sim q(\cdot|x)} \left[ \log p_{\theta}(x|z) \right].
\end{align*}
This concludes the derivation.

\section{KL Divergence}

\begin{lemma}[KL divergence of multivariate normals]\label{lem:kl-normal}
    Suppose both $p$ and $q$ are the pdfs of multivariate normal distributions $\cN_d(\bmu_1, \bSigma_1)$ and $\cN_d(\bmu_2, \bSigma_2)$, respectively. 
    The Kullback-Leibler distance from $p$ to $q$ is given by
    \begin{align*}
        & D_\text{KL}(p ||q) =  \frac{1}{2}\left[\log\frac{|\bSigma_2|}{|\bSigma_1|} - d + \tr \{ \bSigma_2^{-1}\bSigma_1 \} \right. \\
        & \hspace{10ex} + \left. (\bmu_2 - \bmu_1)^T \bSigma_2^{-1}(\bmu_2 - \bmu_1)\right].
    \end{align*}
    When $\bSigma_1=\diag(\sigma_1^2,\ldots,\sigma_d^2)$ and $\bSigma_2=\bI_d$, it reduces to
    \begin{align*}
        D_\text{KL}(p ||q) =  - \frac{1}{2}\left[ d + \sum_{j=1}^d\log \sigma_j^2 - \sum_{j=1}^d \sigma_j^2 - \|\bmu_2 - \bmu_1\|_2^2\right].
    \end{align*}
\end{lemma}
\begin{proof}[Proof of \Cref{lem:kl-normal}]
    
By definition, we have
\begin{align*}
    D_\text{KL}(p ||q)
    & =\int \left[\log( p(\bx)) - \log( q(\bx)) \right]\ p(\bx)\rd \bx \\
    & = \int \left[ \frac{1}{2} \log\frac{|\bSigma_2|}{|\bSigma_1|} - \frac{1}{2} (x-\bmu_1)^T\bSigma_1^{-1}(x-\bmu_1) \right. \\ 
    & \quad + \left. \frac{1}{2} (x-\bmu_2)^T\bSigma_2^{-1}(x-\bmu_2) \right] p(x) \rd x \\
    & = \frac{1}{2} \log\frac{|\bSigma_2|}{|\bSigma_1|} - \frac{1}{2} \tr\ \left\{\EE[(x - \bmu_1)(x - \bmu_1)^T] \ \bSigma_1^{-1} \right\} \\
    & \quad + \frac{1}{2} \EE[(x - \bmu_2)^T \bSigma_2^{-1} (x - \bmu_2)] \\
    & \hspace{-5ex} = \frac{1}{2} \log\frac{|\bSigma_2|}{|\bSigma_1|} - \frac{1}{2} \tr\ \{I_d \} + \frac{1}{2} (\bmu_1 - \bmu_2)^T \bSigma_2^{-1} (\bmu_1 - \bmu_2) \\
    & \hspace{-2ex} + \frac{1}{2} \tr \{ \bSigma_2^{-1} \bSigma_1 \}
\end{align*}
The last expression on the right-hand side translates to:
\begin{align*}
    \frac{1}{2}\left[\log\frac{|\bSigma_2|}{|\bSigma_1|} - d + \tr \{ \bSigma_2^{-1}\bSigma_1 \} + (\bmu_2 - \bmu_1)^T \bSigma_2^{-1}(\bmu_2 - \bmu_1)\right],
\end{align*}
which finishes the proof.
\end{proof}

\section{Additional Experiment Details}
\label{app:additional-exp}

In this section, we show an additional earlier version of the experiment result evaluating the spatial autocorrelation and isometry preserving property of our proposed \dpVAE model. Below, we introduce the two main evaluation metrics used.

\begin{itemize}
    \item Moran's I \citep{li2007beyond} is a correlation coefficient that measures the overall spatial autocorrelation of a dataset. Intuitively, for a given gene, it measures how one spot is similar to others surrounding it. If the spots are attracted (or repelled) by each other, it implies they are not independent. Thus, the presence of autocorrelation indicates the spatial pattern of gene expression. Moran's I value ranges from –1 to 1, where a value close to 1 indicates a clear spatial pattern, a value close to 0 indicates random spatial expression, and a value close to –1 indicates a chessboard-like pattern. To evaluate the spatial variability of a given gene, we calculate Moran's I using the following formula,
    \begin{align*}
        I=\frac{N}{W} \frac{\sum_i \sum_j\left[w_{i j}\left(x_i-\bar{x}\right)\left(x_j-\bar{x}\right)\right]}{\sum_i\left(x_i-\bar{x}\right)^2},
    \end{align*}
    where $x_i$ and $x_j$ are the gene expression of spots $i$ and $j$, $\bar{x}$ is the mean expression of the feature, $N$ is the total number of spots, $w_{ij}$ is the spatial weight between spots $i$ and $j$ calculated using the 2D spatial coordinates of the spots, and $W$ is the sum of $w_{ij}$. We select the $k$ nearest neighbors for each spot using spatial coordinates. Moran's I statistic is robust to the choice of $k$ and is set at 5 in our analysis.
    We assign $w_{ij}=1$ if spot $j$ is in the nearest neighbors of spot $i$, and 0 otherwise.

    \item Geary's $C$ is another commonly used statistic for measuring spatial autocorrelation. It is calculated as:
    $$
    C=\frac{N - 1}{2 W} \frac{\sum_i \sum_j\left[w_{i j}\left(x_i-x_j\right)^2\right]}{\sum_i\left(x_i-\bar{x}\right)^2},
    $$
    The value of Geary's $C$ ranges from zero to two, where zero indicates perfect positive autocorrelation. 
\end{itemize}

{\noindent \bf Evaluation Protocals.}
Moran's I \citep{li2007beyond} is a correlation coefficient that measures the overall spatial autocorrelation of a dataset. Intuitively, for a given gene, it measures how one spot is similar to others surrounding it. If the spots are attracted (or repelled) by each other, it implies they are not independent. Thus, the presence of autocorrelation indicates the spatial pattern of gene expression. Moran's I value ranges from –1 to 1, where a value close to 1 indicates a clear spatial pattern, a value close to 0 indicates random spatial expression, and a value close to –1 indicates a chessboard-like pattern. To evaluate the spatial variability of a given gene, we calculate Moran's I using the following formula,
\begin{align*}
    I=\frac{N}{W} \frac{\sum_i \sum_j\left[w_{i j}\left(x_i-\bar{x}\right)\left(x_j-\bar{x}\right)\right]}{\sum_i\left(x_i-\bar{x}\right)^2},
\end{align*}
where $x_i$ and $x_j$ are the gene expression of spots $i$ and $j$, $\bar{x}$ is the mean expression of the feature, $N$ is the total number of spots, $w_{ij}$ is the spatial weight between spots $i$ and $j$ calculated using the 2D spatial coordinates of the spots, and $W$ is the sum of $w_{ij}$. We select the $k$ nearest neighbors for each spot using spatial coordinates. Moran's I statistic is robust to the choice of $k$ and is set at 5 in our analysis.
We assign $w_{ij}=1$ if spot $j$ is in the nearest neighbors of spot $i$, and 0 otherwise.

Geary's $C$ is another commonly used statistic for measuring spatial autocorrelation. It is calculated as:
$$
C=\frac{N - 1}{2 W} \frac{\sum_i \sum_j\left[w_{i j}\left(x_i-x_j\right)^2\right]}{\sum_i\left(x_i-\bar{x}\right)^2},
$$
The value of Geary's $C$ ranges from zero to two, where zero indicates perfect positive autocorrelation. 

\begin{figure*}[t]
    \centering
    \includegraphics[width=0.49\textwidth]{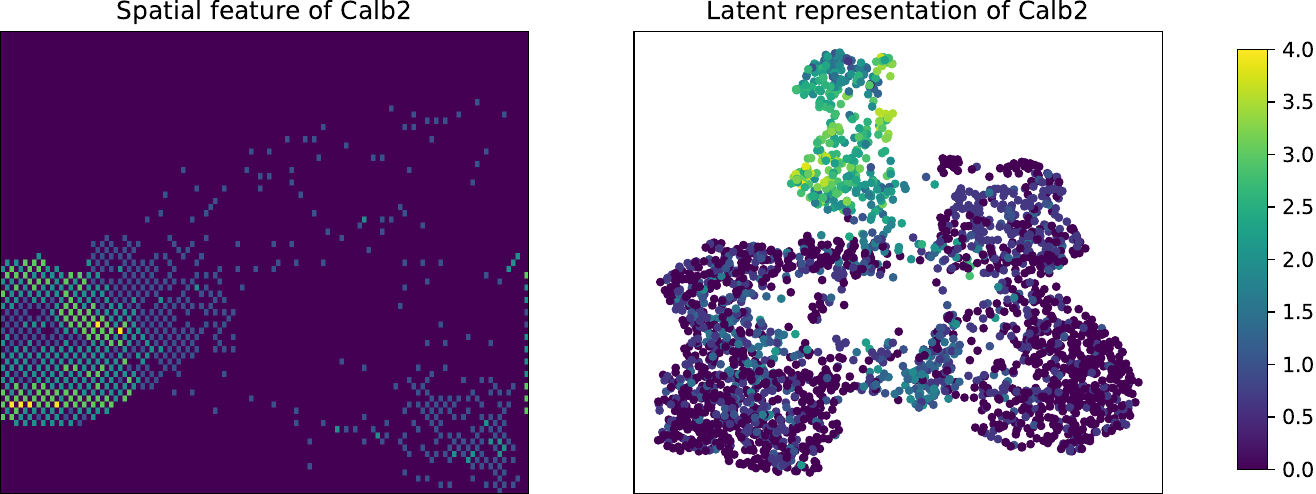}
    \vspace{1ex}
    \includegraphics[width=0.49\textwidth]{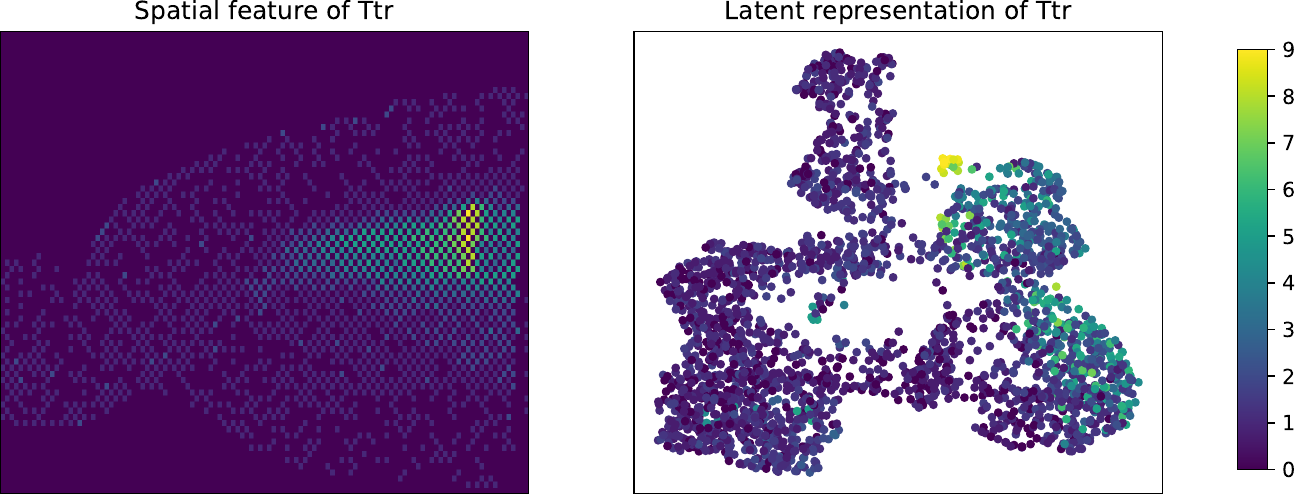}
    \vspace{1ex}
    \includegraphics[width=0.49\textwidth]{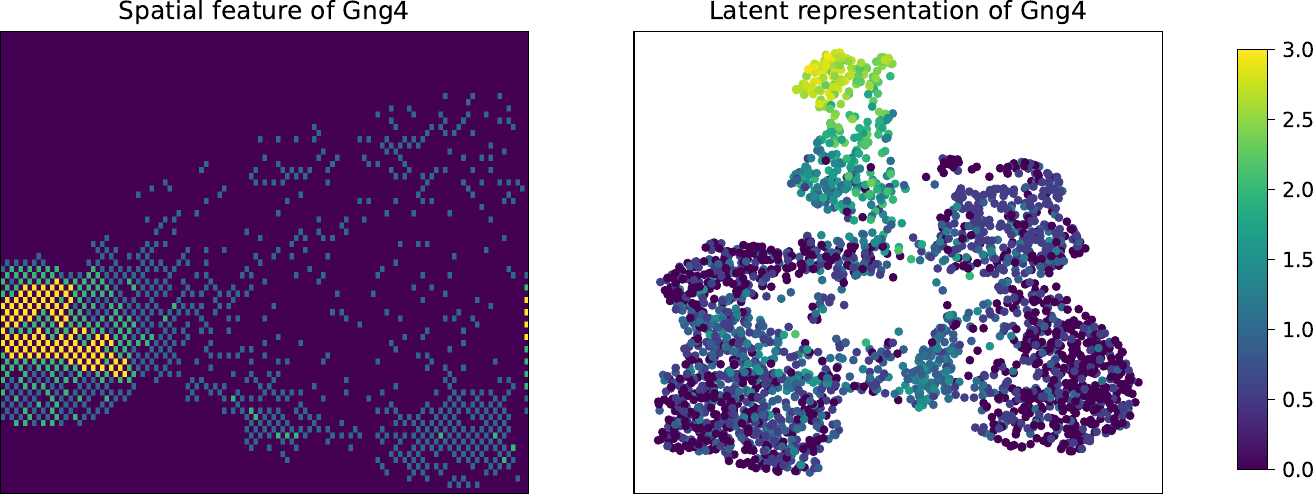}
    \vspace{1ex}
    \includegraphics[width=0.49\textwidth]{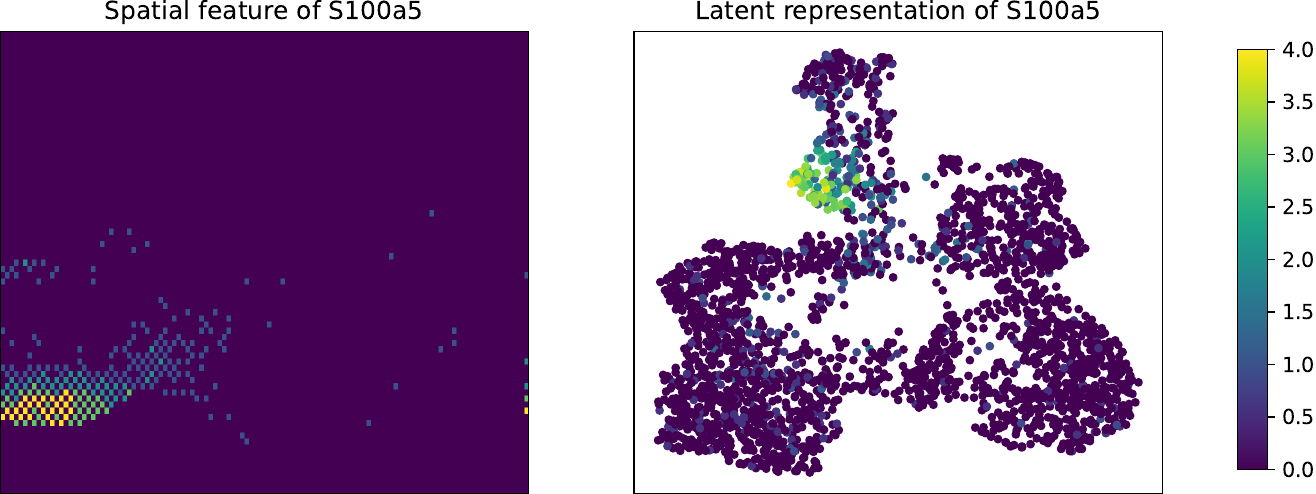}
    \vspace{1ex}
    \includegraphics[width=0.49\textwidth]{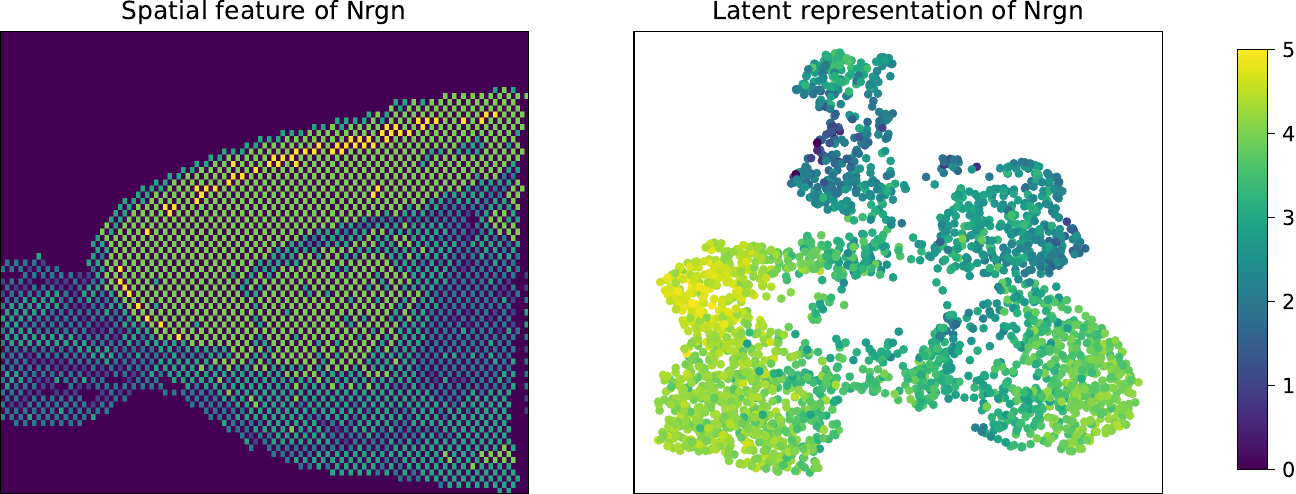}
    \vspace{1ex}
    \includegraphics[width=0.49\textwidth]{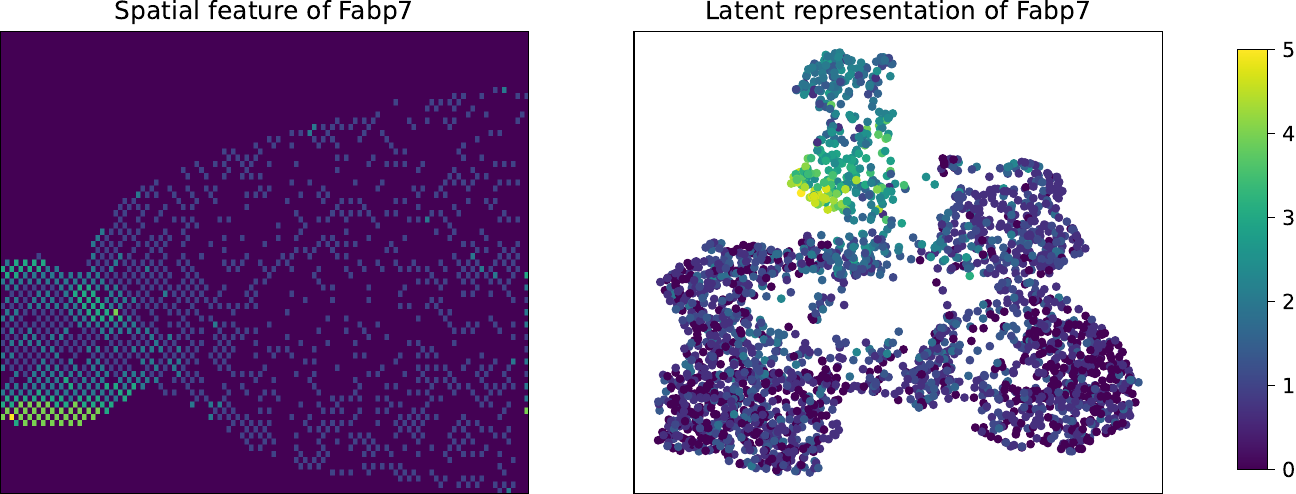}
    \caption{Latent representation extracted from VAE trained on the Mouse Brain Serial Datasets anterior 1. The olfactory bulb region, with genes Calb2, Gng4, S100a5, and Febp7 enriched, is well separated from other regions}
    \label{fig:Z-vae}
\end{figure*}

\begin{figure}[t]
    \centering
    \includegraphics[width=1.0\linewidth]{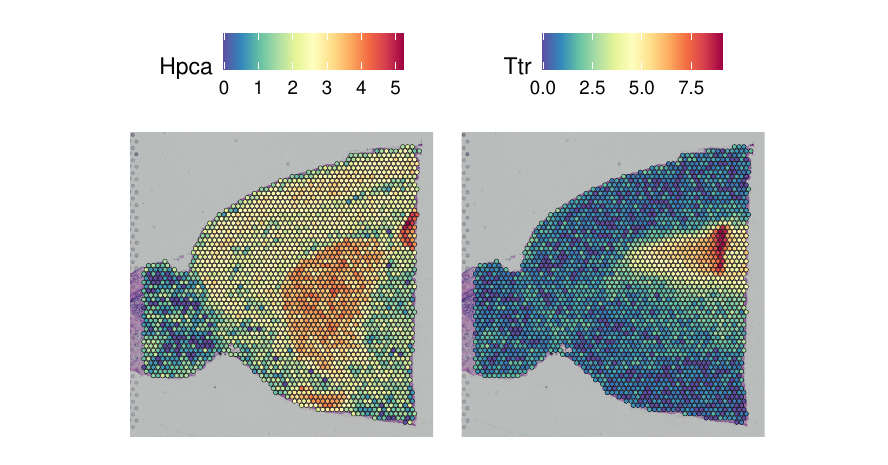}
    \caption{10-lengthed-genes aggregated expression heatmap of the mouse brain cell. Gene expressions show strong spatial trends and local correlations.}
\end{figure}

The experiment results are shown in \Cref{fig:Z-vae}, \Cref{fig:latent}, \Cref{tab:autocorr-results}. 
Figure \ref{fig:Z-vae} shows the gene expression levels of six marker genes extracted by VAE.
For each gene in each panel, the left subplot shows the gene expression levels in the spatial coordinates, and the right subplot shows the 2D-UMAP of the latent representation of the corresponding model. The cells are dispersed in the latent space. The olfactory bulb region, with genes Calb2, Gng4, S100a5, and Febp7 enriched, is well separated from other regions.
\Cref{tab:autocorr-results} shows that enforcing the distance-preserving property induces stronger spatial autocorrelations on out-of-sample datasets.
\Cref{fig:latent} shows that the proposed method can improve the isometric property of the learned latent space. 

\begin{figure}[t]
    \centering
    \includegraphics[width = 0.9\linewidth]{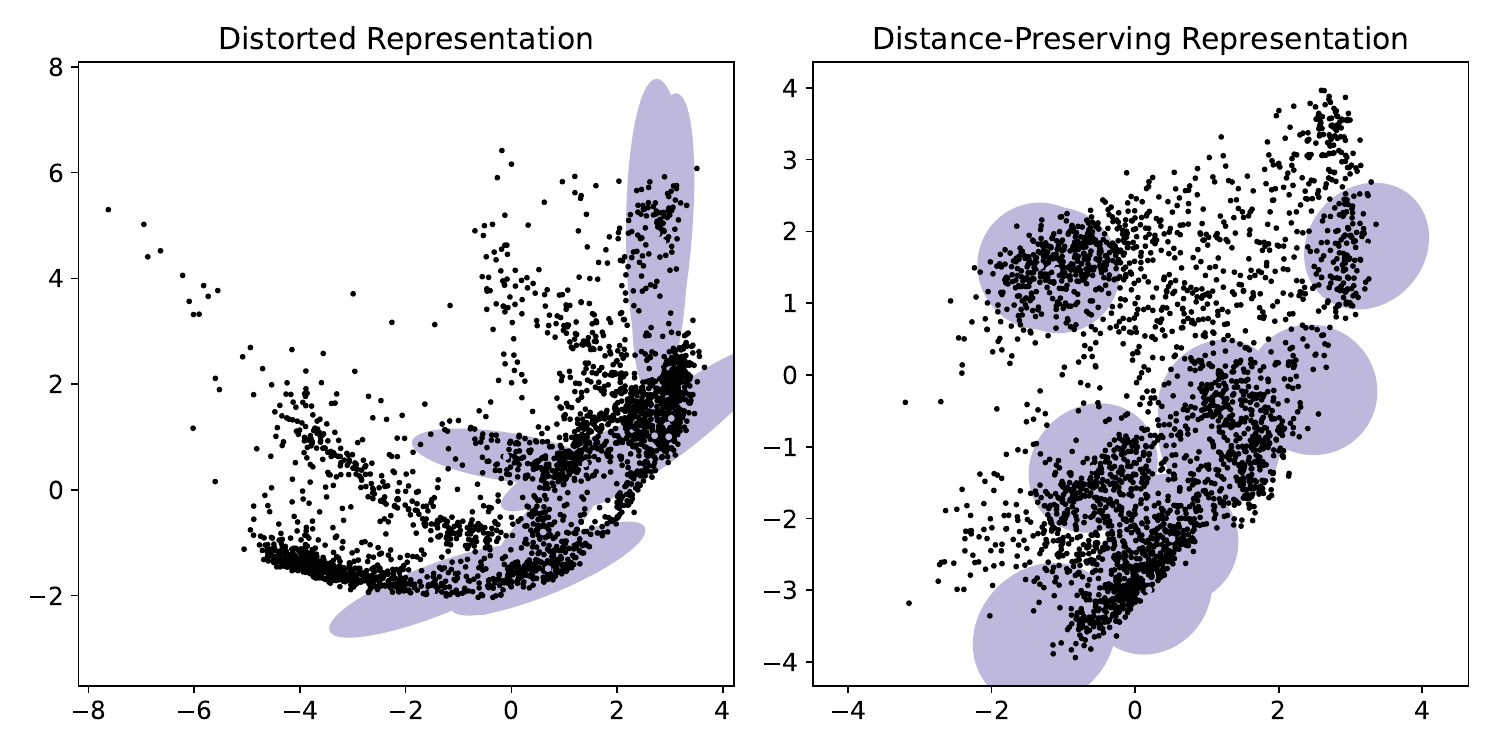}
    \caption{Visualization of latent representation space obtained from scVI (left) and scVI regularized with distance-preserving loss (right).
    More isotropic and homogeneous ellipses indicate more distance-preserving \cite{yonghyeon2021regularized}.}
    \label{fig:latent}
\end{figure}

\begin{table*}
    \centering
    \caption{Moran's I and Geary's C of the latent representation extracted by scVI and VAE using 4 mouse brain cell test datasets, with and without distance-preserving penalty, averaged over 5 repeated trials to account for the randomness of the training process.}
    \begin{tabular}{cccccccccc}
        \toprule
        \midrule
        & \multirow{2}{*}{Method} & \multicolumn{4}{c}{Moran's I} & \multicolumn{4}{c}{Geary's C} \\

        \cmidrule(lr){3-6} \cmidrule(lr){7-10}
        
        & & A2 & A1 & P2 & P1 & A2 & A1 & P2 & P1 \\

        \midrule

        \multirow{2}{*}{Vanilla}
        & VAE   & 0.62(0.07) & 0.55(0.05) & 0.52(0.05) & \textbf{0.52(0.03)} & 0.36(0.06) & 0.41(0.03) & 0.49(0.05) & \textbf{0.43(0.03)} \\
        & scVI  & 0.43(0.03) & 0.52(0.04) & 0.37(0.02) & 0.45(0.04) & 0.57(0.03) & 0.48(0.04) & 0.62(0.02) & 0.55(0.04) \\
        
        \midrule

        \multirow{2}{*}{\shortstack{Distance \\ Preserving  }}
        & VAE   & \textbf{0.64(0.02)} & \textbf{0.60(0.03)} & \textbf{0.56(0.06)} & 0.49(0.06) & \textbf{0.35(0.02)} & \textbf{0.37(0.02)} & \textbf{0.45(0.07)} & 0.46(0.05) \\
        & scVI  & \textbf{0.45(0.04)} & 0.52(0.04) & \textbf{0.43(0.02)} & \textbf{0.47(0.03)} & \textbf{0.55(0.05)} & 0.48(0.04) & \textbf{0.57(0.02)} & \textbf{0.53(0.03)} \\
         \midrule
         \bottomrule
    \end{tabular}
    \label{tab:autocorr-results}
\end{table*}

{
\color{black}
{\noindent \bf Transfer Learning Setup.}
Across all of our experiments, we used highly variable genes (HVGs), meaning that they were selected based on variance in gene expression across cells. We chose HVGs over spatially variable genes (SVGs) because HVGs capture the most informative expression patterns for representation learning, and they are more likely to be consistently expressed across different datasets and brain regions.

The adequate expression of these selected genes across both training and test datasets is ensured by the design of our experiment, where we intentionally used ``similar'' datasets as training and testing pairs. Specifically, the four selected datasets -- abbreviated as VMBSP, VMBSPS2, VMBSA, and VMBSAS2 -- all come from the V1 mouse brain sagittal area. This similarity ensures that the overall gene expression patterns are also highly similar across the four datasets, including those of the highly variable genes.
In fact, a total of $74\%$ out of the top 100 HVGs that were selected appear across all four datasets, with the proportion increasing further when only considering pairwise overlaps, as shown in \Cref{tab:pairwise-overlap}:

\begin{table}[H]
    \caption{Pairwise overlap ratio of shared genes across datasets}
    \centering
    \begin{adjustbox}{max width=0.9\linewidth}
    \begin{threeparttable}
    \begin{tabular}{ c c c c c } 
        \toprule[1pt]
        \textbf{Dataset} & \texttt{VMBSP} & \texttt{VMBSPS2} & \texttt{VMBSA} & \texttt{VMBSAS2} \\
        \midrule
        \texttt{VMBSP}   & \color{gray} $100\%$ & $96\%$ & $77\%$ & $79\%$ \\
        \texttt{VMBSPS2} & \color{gray} $96\%$ & \color{gray} $100\%$ & $78\%$ & $81\%$ \\
        \texttt{VMBSA}   & \color{gray} $77\%$ & \color{gray} $78\%$ & \color{gray} $100\%$ &  $92\%$ \\
        \texttt{VMBSAS2} & \color{gray} $79\%$ & \color{gray} $81\%$ & \color{gray} $92\%$ & \color{gray} $100\%$ \\
        \bottomrule[1pt]
    \end{tabular}
    \end{threeparttable}
    \end{adjustbox}
    \label{tab:pairwise-overlap}
\end{table}
Finally, the top 100 HVGs that are shared across the four datasets are selected, and we aligned them in the same order to form our dataset vectors.
Therefore, this procedure ensures HVGs are adequately expressed in our experiments. 
}

{
\color{black}
\section{Spatial Reconstruction Error}

This section describes the detailed procedure for computing the spatial reconstruction error defined in \Cref{sec:exp-setup}.

Recall that the spatial reconstruction error is defined as solving the following optimization problem:
\begin{align*}
    \min_{\textbf{A}, b, c} \quad & \| c \cdot \textbf{A} \textbf{D}_1 + b - \textbf{D}_2 \|_{F}.
\end{align*}
Here $\textbf{A} \in \mathbb{R}^{m, m}$ is an orthogonal matrix, $b \in \mathbb{R}^m$, and $c$ is a scalar.
The objective value is equivalent ot the Procrustes distance between datasets $\mathbf{D}_1, \mathbf{D}_2 \in \mathbb{R}^{n \times m}$.
The optimization problem is solved mainly via singular value decomposition (SVD). Specifically:
\begin{enumerate}
    \item Both the rescaling term $c$ and translation term $b$ are determined through the normalization, where the former is set to the ratio of their Forbenius norm $\| \mathbf{D}_2 \|_F / \| \mathbf{D} \|_f$, and the latter is set to the difference of their centroid $(\mathbf{1}^\top \mathbf{D}_2 - \mathbf{1}^\top \mathbf{D}_1) / n$.
    \item Then, the rotation matrix $\mathbf{A}$ is computed via SVD of the cross-covariance matrix of the normalized data $\tilde{\mathbf{D}}_1 := c \cdot \mathbf{D}_1 + b$ and $\mathbf{D}_2$, finding the optimal orthogonal transformation matrix that ``aligns'' the two datasets the best to minimize their Frobenius distance: (i) The cross-covariance matrix is computed as $\Sigma = \tilde{\mathbf{D}}_1^\top \mathbf{D}_2$; (ii) Compute its SVD $\Sigma = U \Lambda V^\top$; (iii) Set $\mathbf{A} = U V^\top$.
\end{enumerate}
In our implementation, we directly use the SciPy Python package \cite{2020SciPy-NMeth}, where $\mathbf{A}$, $b$, and $c$ are all computed implicitly within intermediate procedures.
}

\section{Additional Background of VAE}   

{\noindent \bf CVAE.}
A conditional variational encoder (CVAE) \cite{sohn2015learning, wu2023counterfactual} is an extension of VAE, where an additional conditional input $y \in \mathcal{Y}$ can be specified in the encoder and decoder network
\begin{align*}
    q_\theta : \mathcal{X} \times \mathcal{Y} \to \mathcal{Z}, \\
    p_\phi: \mathcal{Z} \times \mathcal{Y} \to \mathcal{X}.
\end{align*}
Similar to VAE, a typical parameterization of CVAE specifies the encoder and decoder functions as conditional Gaussian distributions, where the mean and variance are parameterized by neural networks.

\noindent\textbf{$\beta$-VAE.}
$\beta$-VAE \cite{higgins2017beta} incorporates a temperature term of the KL loss term in the VAE model:
\[
- \beta \cdot D_\text{KL}(q(z|x) || p_{\theta}(z)) + \EE_{z\sim q(\cdot|x)} \left[ \log p_{\theta}(x|z) \right].
\]
Intuitively, a larger KL weight forces the VAE to confine to the prior distribution more, and a smaller KL weight allows the VAE to focus more on minimizing the reconstruction loss.

\noindent\textbf{scVI (single-cell Variational Inference).}
scVI \cite{lopez2018deep} is a leading deep generative model for analyzing single-cell RNA sequencing (scRNA-seq) data. It employs a CVAE architecture to model gene expression patterns across individual cells. In scVI, the input consists of gene expression counts for each cell, while the model aims to learn a latent representation that captures biological variability. The model accounts for technical factors such as library size and batch effects, which are treated as observed variables. Specifically, scVI uses the normalized gene expression counts as the target variable, with the library size included as a nuisance variable to be inferred. Batch information, when available, is incorporated as conditional information to help disentangle technical from biological variability. This approach allows scVI to perform various downstream tasks, including normalization, batch correction, imputation, and dimensionality reduction, making it a versatile tool for single-cell transcriptomics analysis.

{\bf \noindent Markovian Hierarchical VAE.}
A Markovian Hierarchical Variational Autoencoder (MHVAE) \cite{kingma2016improved, sonderby2016ladder} is a generalization of a VAE that extends to multiple hierarchies over latent variables. Under this formulation, latent variables are interpreted as generated from other higher-level, more abstract patterns. Let \(T > 1\) denote the total number of hierarchical levels; the objective of an MHVAE is similar to VAE, except that there are now \(T\) latent variables whose joint distribution can be decomposed into Markov transitions:
\begin{align}
    p_\theta(x) & = \int p_\theta(x | z_{1:T})d z_{1:T} \\
    & = \int p_\theta(x | z_1) \prod_{t = 2}^T p_\theta(z_t | z_{t+1}) d z_{1:T}.
\end{align}
Deriving the ELBO of the MHVAE can follow a similar procedure to that of VAE:
\begin{align*}
    p_\theta(x)
    & \geq \EE_{q_\phi(z_1 | x)}\left[ \log p_\theta(y | z_1) \right] \\
    & \quad - \EE_{q_\phi(z_{T-1} | x)}\left[ D_{\rm KL}\left( q_\phi\left( z_T | z_{T-1} \right) \Vert p(z_T) \right) \right] \notag \\
    & - \sum_{t = 1}^{T-1} \EE_{q_\phi(z_{t-1}, z_{t+1} | x)} \left[ D_{\rm KL}(q_\phi(z_t | z_{t-1}) \Vert p_\theta (z_t | z_{t+1})) \right].
\end{align*}
with \(z_0 \equiv 0\) for notational convenience.
The first term (reconstruction term) and the second term (prior matching term) can be interpreted similarly. The third term (consistency term) is new. It endeavors to make the distribution at \(z_t \) consistent from both forward and backward processes. All terms are taken with respect to distributions that we can sample from, and Monte Carlo estimation can be obtained.

However, a key caveat is that the expectation over two random variables \( \{z_{t-1}, z_{t+1}\} \) for every timestep, the variance of its Monte Carlo estimate could potentially be higher than a term that is estimated using only one random variable per timestep. As it is computed by summing up \(T-1\) consistency terms, the final estimated value of the ELBO may have a high variance for large \(T\) values. Instead, we can derive an equivalent version that does not pose this issue:
\begin{align}
    p_\theta(y)
    & \geq \EE_{q_\phi(z_1 | x)}\left[ \log p_\theta(x | z_1) \right]
    - D_{\rm KL}\left( q_\phi\left( z_T | x \right) \Vert p(z_T) \right) \notag \\
    & \quad - \sum_{t = 2}^{T-1} \EE_{q_\phi(z_{t} | x)} \left[ D_{\rm KL}(q_\phi(z_t -1 | z_{t}, x) \Vert p_\theta (z_{t-1} | z_{t})) \right].
\end{align}
The modified last term is called {\it denoising matching term} following the variational diffusion model literature \cite{luo2022understanding}.  We learn desired denoising
transition step \( p_\theta(z_{t-1}|z_t)\) as an approximation to tractable, variational denoising transition step \( q_\phi (z_{t-1}|z_t, x) \).

\vfill

\end{document}